\documentclass[10pt]{llncs}
\usepackage[latin1]{inputenc}
\usepackage[T1]{fontenc}
\usepackage{mathrsfs}
\usepackage{amssymb,amsmath,amsfonts}
\usepackage{dsfont}
\usepackage[usenames,dvipsnames]{pstricks}
\usepackage{float}
\usepackage{dsfont}
\usepackage{pifont}
\usepackage{txfonts}
\usepackage{subfigure}
\usepackage{graphicx}
\usepackage{geometry}
\newcommand{\cmark}{\checkmark}
\newcommand{\xmark}{\ding{53}}
\newcommand{\Omit}[1]{}
\newcommand{\Bigwedge}{\bigwedge\!\!\!\!\bigwedge}

\renewcommand{\phi}{\varphi}
\newcommand{\Set}[1]{\{ {#1} \}}
\newcommand{\dans}{\longrightarrow}
\newcommand{\model}[1]{\mbox{$[\hspace{-.38mm}[$}{#1}\mbox{$]\hspace{-.40mm}]$}}

\newcommand{\Form}{\mathcal{L_P}}
\newcommand{\Val}{\mathcal{W_P}}
\newcommand{\ee}{epistemic state}
\newcommand{\ees}{epistemic states}
\newcommand{\sinti}{{\bf(ESF1)}}
\newcommand{\sintii}{{\bf(ESF2)}}
\newcommand{\sintiii}{{\bf(ESF3)}}
\newcommand{\sintiv}{{\bf(ESF4)}}
\newcommand{\sintv}{{\bf(ESF5)}}
\newcommand{\sintvi}{{\bf(ESF6)}}
\newcommand{\sintvii}{{\bf(ESF7)}}
\newcommand{\sintviii}{{\bf(ESF8)}}
\newcommand{\sintviiiw}{{\bf(ESF8W)}}

\newcommand{\sintSD}{{\bf(ESF-SD)}}
\newcommand{\sintU}{{\bf(ESF-U)}}
\newcommand{\sintP}{{\bf(ESF-P)}}
\newcommand{\sintI}{{\bf(ESF-I)}}
\newcommand{\sintD}{{\bf(ESF-D)}}
\newcommand{\assgi}{{\bf 1}}
\newcommand{\assgii}{{\bf 2}}
\newcommand{\assgiii}{{\bf 3}}
\newcommand{\assgiv}{{\bf 4}}

\newcommand{\assgivp}{{\bf 4'}}

\newcommand{\miff}{{\em iff} }
\newcommand{\ie}{{\em i.e.}}
\newcommand{\reals}{\mathds{R}}
\newcommand{\nat}{\mathds{N}}
\newtheorem{observation}{Observation}

\newcommand{\ordre}[1]{$\begin{array}[c]{c} #1 \end{array}$}

\newcommand{\ordred}[2]{${#1}= \begin{array}[c]{c}
#2 \end{array}$}

\begin{document}

\title{\bf Impossibility in Belief Merging}%
\author{ Amílcar {Mata~Díaz}\inst{1} \and Ram\'{o}n {Pino~P\'{e}rez}\inst{2}}
\institute{ \em Departamento de Medición y Evaluación\\
 Facultad de Humanidades y Educación\\
  Universidad de Los Andes\\
   M\'{e}rida, Venezuela\\
\email{amilcarmata@ula.ve}\\[1.5mm]
\and
Departamento de Matemáticas\\
   Facultad de Ciencias\\
   Universidad de Los Andes\\
    M\'{e}rida, Venezuela\\
\email{pino@ula.ve}
}

\maketitle

\begin{center}
\bf June 1, 2016
\end{center}

\begin{abstract}
 With the aim of studying social properties of belief merging and having a better understanding of impossibility, we extend in three ways the framework of logic-based merging introduced by Konieczny and Pino Pérez.
   First, at the level of representation of the information, we pass from  belief bases   to complex epistemic states. Second, the profiles are represented as functions of finite societies to the set of epistemic states (a sort of vectors) and not as multisets of epistemic states. Third, we extend the set of rational postulates in order to consider the epistemic versions of the  classical  postulates of Social Choice Theory: Standard Domain, Pareto Property, Independence of Irrelevant Alternatives and Absence of Dictator. These epistemic versions of social postulates are given, essentially, in terms of the finite propositional logic.   We state some representation theorems for these operators. These extensions and representation theorems allow us to  establish an epistemic and very general version of Arrow's Impossibility Theorem.
  One of the interesting features of our result,  is that it holds for different representations of epistemic states; for instance conditionals, Ordinal Conditional functions and, of course, total preorders.
\end{abstract}

\hspace*{0.45cm}{\sf Keywords:}
Belief merging, epistemic states, Arrow's Impossibility Theorem

\section{Introduction}
 Belief merging studies the methods leading to extract a coherent piece of  information from many sources, which may be mutually contradictory. The applications in this domain go from decision making,
passing by medical diagnosis, policy planning, to automatic integration of data.
Understanding the theoretical model, its properties and limits  is important in order to know in which domains these methods are valid and to develop future applications.

From the early works in belief merging \cite{KP98,KP99,KP02,KLM04,KP05} there has been a feeling of closeness between this framework and
the Social Choice Theory \cite{Ar63,Suz02}. In particular, the way in which concrete merging operators are constructed -via the representation theorem-
evoques the methods for constructing social choice functions.
Some works \cite{M00,CGM06,EKM07} have been done in which
certain aspects of social choice are explored in belief merging, mainly impossibility \cite{Ar63,CK02} and strategy-proofness \cite{Gib73,Sat75}.

In some of the previous works, for instance in Chopra et al. \cite{CGM06}, the results of impossibility do not hold and the operators are strategy-proofness. In others works, like in Everaere et al. \cite{EKM07}, the behavior of some classes of merging operators is studied with respect to the manipulation. Therein is proved that some classes of merging operators are not strategy-proofness. In that work the impossibility problems are not addressed.

Facing these opposed behaviors of some classes of merging operators,
our goal is to better understand  what is happening in terms of general rational properties expressed in the logical language of
belief merging. In particular, we are interested in knowing if there is a general result of impossibility in this setting, such as the
Arrow's Impossibility Theorem in Social Choice Theory.

In order to accomplish our study,  we extend in this work the framework
  of logic-based merging   presented by Konieczny and Pino Pérez (KPP) \cite{KP98,KP99,KP02,KP05,KP11} in three ways. First, the representation of information is more general; actually, we continue with the abstract view of epistemic states present in \cite{MP11}.
  Second, the representation of profiles will be also more general; we adopt a functional view. Third,  we enrich the set of rational postulates by  introducing social postulates inspired on the classical Arrow's postulates in Social Choice Theory \cite{Ar63}. All these postulates are formulated in a logical setting.

  Remember that in the works of Konieczny and Pino Pérez  the information (the beliefs)  is defined by sets of  propositional sentences.  It is important to notice that the KPP
  framework generalizes the seminal belief revision operators presented by  Alchourrón, Gardenfors and Makinson \cite{AGM85,G88,KM92} in which
  the beliefs (alias epistemic states) have also the same kind of representation, that is, a set of propositional sentences.

However, the necessity of considering more complex representations of epistemic states has been stated in the work of Darwiche and Pearl \cite{DP97}, in particular in order to have a good behavior with respect to the iteration of the process. We adopt here, essentially,  the definition proposed by
Benferhat et ál. \cite{BKPP00}.

 We have to note that Meyer \cite{M00}
 gave a merging model of more complex epistemic states for belief merging. In Meyer's work the epistemic states
 were defined by  ranking functions over valuations. However, in this context, Meyer did not develop any study of the logical properties of these operators, and, therefore, does not give representation theorems for his operators. In this work we use an abstract definition of epistemic states that can be instantiated in different concrete representations. For instance, Ordinal Conditional Functions (ranking functions) proposed by Spohn \cite{Spo88}; some kind of conditionals like rational consequence relations proposed by Lehmann and Magidor \cite{LM92}; the total preorders  can be also concrete representations of epistemic states.

The notion of profile, first given as a multiset of belief bases, is extended in the current work. Actually, the multiset representation has implicitly
the anonymity  of the sources of information. Thus, this is not compatible with the existence of a dictator even not with the existence of a manipulator.
In order to have a representation of the profiles which is more flexible and eventually compatible with the existence of dictators and more general than multisets, we adopt the functional view: a profile will be a function of a finite set of agents in the set of epistemic states. Our view generalizes the notion of profiles as  lists or vectors
explicit in \cite{CGM06,EKM07}.

In the current work, we present some postulates given in terms of the finite propositional logic (the ``visible'' part of  epistemic states) which define some classes of merging operators of abstract epistemic states.
We state some key representation theorems for these operators.
In particular, we give the logical versions of classical postulates involved in the Arrow's Impossibility Theorem: Standard Domain, Pareto Property, Independence of Irrelevant Alternatives and Absence of Dictator. As a matter of fact, the postulate equivalent to Transitive Explanations will
be always satisfied by our basic operators. We prove some relationships between these new postulates and the classical postulates of merging. Then we establish a general result of impossibility: The epistemic state merging operators satisfying standard domain and independence are, in fact, dictatorial operators.

It is worth to note that, due to the fact that our notion of epistemic states has rich instantiations, our general Impossibility Theorem has instantiations in rich representations of epistemic states as those mentioned before: Ordinal Conditional Functions, rational consequence relations and total pre-orders.

It is interesting to remark that our approach makes a strong use of the Standard Domain Postulate. However, this postulate is incompatible with the representation of epistemic states as propositional formulas. Actually, we will prove that  such kind of representation is not rich enough in order to satisfy Standard Domain. This result strengthens the utility and the necessity of considering complex epistemic states.

Our Dictator notion is actually a weak notion of dictator: at the level of beliefs, the result of the merging entails the beliefs of the Dictator. It is for this notion that our Impossibility Theorem holds. But this theorem does not hold for a strong version of dictator, namely: that the result of the merging, at the level of beliefs,  be {\em equivalent} to the beliefs of the Dictator. This will be seen with some examples of operators at work.

The Syntax Independence Postulate, (IC3), is not present in its complete form. Actually this postulate implies a very strong form of  anonymity and therefore it is incompatible with a notion of dictator. In our setting, this postulate is slightly weakened in order to render it
compatible with the notion of dictator.

Following the early works of Katsuno and Mendelzon in Belief Revision \cite{KM92}, in which all the pieces of information have the same nature\footnote{For them the old beliefs and the new beliefs are propositional formulas.}, in this work all the pieces of information are (complex) epistemic states.
In particular, the information playing the role of integrity constraints is represented as an epistemic state too. In this way we generalize
the epistemic states revision operators proposed by Benferhat et al. \cite{BKPP00} (see also \cite{MP11}).

We have to say that the Fairness Postulate, (IC4), considered in the KPP framework is not present in this work. The main reason for putting aside this postulate  is its specificity. The sole known merging operators satisfying this postulate are the operators defined by distances (see \cite{KP02,KP11}).
Actually, in \cite{KP02} is proved that some very natural operators satisfy (IC4) iff they are defined using a distance. There is, also, another important reason for withdrawing (IC4):  many important classes of well behaved operators that we believe are worth being considered, do not satisfy (IC4). The following example illustrates this situation.

\begin{example}
Anne and Bob have to travel from point $A$ to point $B$. There are four paths to perform the travel: $w_1$, $w_2$, $w_3$ and $w_4$.
Anne thinks that the best way to carry out the travel is $w_1$. Actually she thinks that $w_i$ is better than $w_{i+1}$ for $i=1,2,3$.
Bob thinks that the better way  to do the travel is $w_4$ and the others possibilities are equally good. A common consensual result in this situation
is to take the  path $w_1$, the path that Anne prefers and which is not so bad for Bob. Whereas the path $w_4$ is very far from Anne's preferred
path.
\end{example}

As a matter of fact, in order to have a better road map of the postulates here considered and their relationships, we  consider concrete  spaces, namely those where the epistemic states are total pre-orders over valuations, and we define some examples of operators for which the satisfaction of postulates is analyzed. One important thing that our examples reveal,  is that neither the Standard Domain Postulate, nor Independence Postulate (nor even Unanimity)
are necessary conditions in order to have dictatorial operators.

This work is organized as follows: Section \ref{preliminaries} is devoted to defining  the concepts used throughout the paper. Section \ref{operators-def} is devoted to giving
the syntactical postulates and their semantical counterparts and then to establish the basic representation theorems. In Section \ref{Impossibility} we introduce the new postulates coming from Social Choice and state  a general Impossibility Theorem, the main result of the paper.  In Section \ref{examples}, we construct operators in a concrete class of epistemic states.
Finally we make some concluding remarks in Section \ref{conclu}.
The proofs are in   \ref{apendice}.

\section{Preliminaries}\label{preliminaries}

Let $A$ be a set. A binary relation $\succeq$ over $A$ is a total preorder if it is total (therefore reflexive) and transitive.
Let $\succeq$ be a total preorder over $A$. We define the strict relation  $\succ$ and the indifference $\simeq$ associated to $\succeq$ as follows: $
a\succ b\mbox{ \miff }a\succeq b\mbox{ and }b\not\succeq a
$;
$
a\simeq b\mbox{ \miff }a\succeq b\mbox{ and }b\succeq a
$.
It is clear that any linear order is a total preorder.

Let $\succeq$ be a total preorder over $A$ and $C$ be  a subset of $A$. We say that $c$ is a maximal element of $C$ with respect to $\succeq$ if $c\in C$ and for all $x\in C,\;x\succeq c$. The set of maximal elements of $C$ will be denoted $\max(C,\succeq)$. The maximal elements of the whole set $A$ will be denoted $\max(\succeq)$. We will denote by $\succeq\upharpoonright_C$ the restriction of $\succeq$
to $C$. The set of total preorders over a set $A$ will be denoted by $\mathds{P}(A)$.

Let us introduce an important example of total preorder:  the lexicographical combination of two total preorders.
 Let  $A$ be a nonempty set and let $\succeq_1$ and $\succeq_2$  be two total preorders   over $A$. We define $\succeq^{\mathrm{lex}(\succeq_1,\succeq_2)}$ over $A$ by putting:

$$
a\succeq^{\mathrm{lex}(\succeq_1,\succeq_2)} b \Leftrightarrow\left\{
\begin{array}{lcl}
a\succ_1 b, & \mbox{ or } &\\
a\simeq_1 b & \& & a\succeq_2 b
\end{array}
\right.
$$
It is not hard to see that $\succeq^{\mathrm{lex}(\succeq_1,\succeq_2)}$  is a total preorder over $A$ and that the following equality holds
\begin{equation}
\max(\succeq^{\mathrm{lex}(\succeq_1,\succeq_2)})=\max(\max(\succeq_1),\succeq_2)
\label{eq:0}
\end{equation}
Moreover, it is also easy to show that if either $\succeq_1$ or $\succeq_2$, is a linear order then $\succeq^{\mathrm{lex}(\succeq_1,\succeq_2)}$ is also a linear order.

Sometimes it is useful to associate some ranking functions to the total preorders over a set $A$. These are functions $f:A\dans \reals$ determining the total preorder $\succeq$ over $A$ by $a\succeq b$ if, and only if $f(a)\geq f(b)$. Actually, we define the {\em canonical ranking function}, $r_\succeq$, associated to a total preorder $\succeq$ over $A$ by putting
$$r_\succeq(a)= \max\{n\in\nat:\exists\; a_0,a_1,\dots,a_n\in A;\; a_{i+1}\succ a_i \mbox{  and  }a_n=a\}$$

 Let $\Form$ be the set of propositional formulas built over a finite set $\mathcal{P}$ of atomic propositions.
 $\Form^\ast$ will denote the set of non contradictory formulas. $\Form^\ast/\!\equiv$ will denote the set of non contradictory formulas modulo logical equivalence.
 Let $\Val$ be the set of valuations of formulas in $\mathcal{L_P}$. If
 $\phi$ is a formula in $\Form$, we denote by $\model\phi$ the set of its models, \ie\ $\model \phi=\{w\in\mathcal{W_P}: w\models\phi\}$. If $\phi_i$ is a formula in $\mathcal{L_P}$, for each $i\in I$ (where $I$ is a finite set of indexes), we denote by $\Bigwedge_{i\in I}{\phi_i}$ the conjunction of all the formulas $\phi_i$ with $i$ in $I$. When $I$ is the empty set, $\Bigwedge_{i\in I}{\phi_i}\equiv \top$. If  $M$ is a nonempty set of valuations, we denote by $\phi_M$ a formula such that $\model{\phi_M}=M$.

In this work a {\em belief base} will be represented by a formula in $\Form$. It encodes the set of propositions believed by an agent. The intuitive meaning of a (complex) {\em epistemic state} is to have, in addition to a belief base, other information, eventually null. Many concrete representations of epistemic states have been proposed. The first one is, of course,
that of the AGM framework  \cite{AGM85,G88}, where an \ee\ is a logical theory. Darwiche and Pearl \cite{DP94,DP97} represent \ees\  by total preorders over $\Val$.
Spohn \cite{Spo88} uses Ordinal conditional functions. An abstract model to  represent \ees\ was presented by Benferhat et al. \cite{BKPP00}. We will adopt this abstract representation throughout this paper except in Section \ref{examples}, where we consider  the concrete representation of total preorders over valuations as a realization of the abstract model.

\begin{definition}[Epistemic space]
 A triple $(\mathcal{E},B,\mathcal{L_P})$ is called an epistemic space if $\mathcal{E}$ is a nonempty set, $\mathcal{L_P}$ is the set of formulas over a set of propositional variables $\mathcal{P}$ and $B$ is a function from $\mathcal{E}$ into $\mathcal{L_P}$, such that the image of $B$ modulo logical equivalence is  all the set $\mathcal{L_P^*}/\!\equiv$.
\end{definition}

The elements of $\mathcal{E}$ are called  {\em \ees}; $B$ is called a {\em belief function}; for any $E$ in $\mathcal{E}$, $B(E)$ is called the {\em entrenchment beliefs} (or belief base) of $E$.
 Notice that if $I$ is a nonempty set of valuations, there is $E$ such that $B(E)\equiv\phi_I$, since the image of $B$ (modulo logical equivalence) is all the set $\mathcal{L_P^*}/\!\equiv$.

As an example of an epistemic space we can consider a finite propositional language $\mathcal{L_P}$, and then let $\mathcal{E}$ be the set of all total preorders over the interpretations, that is $\mathcal{E}=\mathds{P}(\mathcal{W_P})$, and let $B$ be the function that maps an epistemic state $\succeq$ in a formula $\varphi_\succeq$ having as models the maximal elements of $\succeq$.

 In order to introduce the notion of {\em profile}, we consider an epistemic space $(\mathcal{E},B,\mathcal{L_P})$ and a well ordered set $(\mathcal{S}, <)$
 the elements\footnote{The set $(\mathcal{S}, <)$ can be identified with the set of natural numbers $\mathds{N}$ with the usual order.} of which are called {\em agents}. Let $\mathcal{F}^\ast(\mathcal{S})$ be the set of all nonempty finite subsets of agents.
Each set $N$ in  $\mathcal{F}^\ast(\mathcal{S})$ can be seen as a {\em finite society of agents}. We can suppose that if $N=\{i_1,i_2,\dots,i_n\}$ the elements are listed in increasing way, \ie\ $i_k<i_m$ whenever $k<m$.

 Given $N$ in $\mathcal{F}^\ast(\mathcal{S})$, a partition of $N$ is a finite family  $\Set{N_1,N_2,\dots,N_k}$ contained in $\mathcal{F}^\ast(\mathcal{S})$  of  pairwise disjoint sets such that their union is $N$.

\begin{definition}[Profile]
      Given $N$ in $\mathcal{F}^\ast(\mathcal{S})$, an $N$-profile  is a function $\Phi:N\dans \mathcal{E}$. We think of $\Phi(i)$ as the epistemic state of the agent $i$,  for each agent $i$ in $N$.
\end{definition}

 Given a $N$-profile, $\Phi:N\rightarrow\mathcal{E}$, for each agent  $i$ in $N$, $E_i$ will denote $\Phi(i)$. Thus, if $N=\{i_1,i_2,\dots,i_n\}$ is a finite society of agents, it can be seen as an ordered tuple:
$\Phi=(E_{i_1},E_{i_2},\dots,E_{i_n})$. When $N$ is a singleton, suppose $N=\Set i$, by abuse, $E_i$ will denote the $N$-profile $\Phi=(E_i)$.
  In that case we write $i$-profile instead of $\Set i$-profile.
  Given $N$ in $\mathcal{F}^\ast(\mathcal{S})$, the set of $N$-profiles will be denoted $\mathcal{E}^N$. The set of all profiles of epistemic states will be denoted $\mathcal{P}(\mathcal{S,E})$, that is,
$\mathcal{P}(\mathcal{S,E})=\bigcup_{N\in\mathcal{F}^\ast(\mathcal{S})}\mathcal{E}^N$.

Let $N=\{i_1,i_2,\dots,i_n\}$ and $M=\{j_1,j_2,\dots,j_m\}$ be two finite societies of agents in $\mathcal{S}$.  Consider an $N$-profile \linebreak$\Phi=(E_{i_1},E_{i_2},\dots,E_{i_n})$ and an $M$-profile $\Psi=(E_{j_1},E_{j_2},\dots,E_{j_m})$.
We say that $\Phi$ and $\Psi$ are equivalent, denoted $\Phi\equiv \Psi$, if $n=m$ and $E_{i_k}=E_{j_k}$ for $k=1,\dots, n$. If $\Phi$ and $\Psi$ are not equivalent, we write $\Phi\not\equiv \Psi$. From this, for any pair of agents $i$, $j$ in $\mathcal{S}$, if we consider an $i$-profile $E_i$ and $j$-profile $E_j$, then $E_i\equiv E_j$ iff, seen as  epistemic states, we have $E_i=E_j$. Thus, by abuse and being clear from the  context, we write respectively $E_i= E_j$ and $E_i\neq E_j$ instead of $E_i\equiv E_j$ and $E_i\not\equiv E_j$. If $N$ and $M$ are disjoint we define a new
$(N\cup M)$-profile, the joint of
$\Phi$ and $\Psi$, denoted  $\Phi\sqcup\Psi$, in the following way:
$$\big(\Phi\sqcup\Psi\big)(i)=\left\{\begin{array}{lr}
\Phi(i), & \mbox{if } i\in N\\
\Psi(i), & \mbox{if } i\in M\\
\end{array}\right.$$
If $M\subseteq N$, and $M$ is nonempty, and $\Phi$ is an $N$-profile, then $\Phi\upharpoonright_{_M}$ will denote the  $M$-profile obtained by the restriction of $\Phi$ to $M$, that is, $\Phi\upharpoonright_{_M}:M\rightarrow\mathcal{E}$, where $\Phi\upharpoonright_{_M}(j)=\Phi(j)$, for each $j$ in $M$.

 Thus, if $\Set{N_1,N_2,\dots,N_k}$ is a partition of a finite society $N$ of agents in $\mathcal{S}$, for each $N$-profile $\Phi$, we have
$$\Phi=\Phi\upharpoonright_{_{N_1}}\sqcup\cdots\sqcup\Phi\upharpoonright_{_{N_k}}$$

From now on, we will suppose that $N$ is the finite society $\{i_1,i_2,\dots,i_n\}$, while  the $N$-profiles $\Phi$ and $\Phi'$ will be denoted by $(E_{i_1},E_{i_2},\dots,E_{i_n})$
 and $(E_{i_1}',E_{i_2}',\dots,E_{i_n}')$ respectively.

\section{Epistemic states fusion operators}\label{operators-def}

\subsection{Postulates}
We fix an epistemic space $(\mathcal{E},B,\mathcal{L}_{\mathcal{P}})$ and a set of agents $\mathcal{S}$.
A function of the form $\nabla:\mathcal{P}(\mathcal{S,E})\times\mathcal{E}\longrightarrow\mathcal{E}$  will be called an {\em epistemic state  combination operator}, for short
an ES combination operator. $\nabla(\Phi,E)$ represents
 the result of combining the epistemic states in  $\Phi$ under the {\em integrity constraint} $E$.

Now we establish the rationality postulates of fusion in the setting of \ees. Most of them are adapted from IC merging postulates proposed by Konieczny and Pino Pérez \cite{KP99,KP02,KP05} (a first version of these postulates appeared in \cite{MP11}). Since the beliefs of epistemic states
constitute their only aspect with well known logical structure, we express the rationality postulates in logical terms at the level of beliefs.
Some of these postulates are new mainly due to the new presentation of profiles as functions of finite societies into \ees.

In order to introduce such postulates, let $N$ and $M$ be any pair of finite societies of agents in $\mathcal{S}$, $\Phi$ be any $N$-profile, $\Phi'$ be any $M$-profile,  $E$, $E'$, $E''$ be any triple of \ees.

\begin{description}

\item[(ESF1)] {\em $B\big(\nabla(\Phi,E)\big)\vdash B(E)$.}

\item[(ESF2)]  {\em If $\Phi\equiv\Phi'$ and
$B(E)\equiv B(E')$ then $B\big(\nabla(\Phi,E)\big)\equiv B\big(\nabla(\Phi',E')\big)$}

\item[(ESF3)] {\em If $B(E)\equiv B(E')\wedge B(E'')$ then $B\big(\nabla(\Phi,E')\big)\wedge B(E'')\vdash B\big(\nabla(\Phi,E)\big)$.}

\item[(ESF4)] {\em If \mbox{$B(E)\equiv B(E')\wedge B(E'')$} and \mbox{$B\big(\nabla(\Phi,E')\big)\wedge B(E'')\not\vdash\bot$}, then $B\big(\nabla(\Phi,E)\big)\vdash B\big(\nabla(\Phi,E')\big)\wedge B(E'')$.}

\end{description}


 \textbf{(ESF1)} tells us that the belief  of the result has to be logically stronger than the belief of the restriction. This Postulate corresponds to Postulate (IC1) of IC  merging operators \cite{KP02,KP11}.

 \textbf{(ESF2)} is a weak form of the {\em anonymity} at the profile level  and a {\em syntax irrelevance property} at the level of beliefs for the integrity restrictions. This Postulate corresponds\footnote{Actually, the version of equivalence with the profiles as multisets would correspond  in our setting to allowing permutations in the profiles, but indeed we do not authorize it.} to Postulate (IC3) of IC merging operators \cite{KP02,KP11}.

 \textbf{(ESF3)} and \textbf{(ESF4)} together  determine one important property in which the beliefs are chosen. They correspond to postulates (IC7) and (IC8) of IC merging operators respectively \cite{KP02,KP11}.

These four postulates, \sinti-\sintiv, are called  {\em basic epistemic state fusion postulates}. They are considered the minimal requirements of rationality that
the combination operators have to satisfy.

\begin{definition}[Basic fusion operators]
 Let $(\mathcal{E},B,\mathcal{L}_{\mathcal P})$ be  an epistemic space and let $\mathcal{S}$ be a set of agents. A  combination operator of epistemic states $\nabla$ is said to be an  epistemic state  basic fusion operator  (ES basic fusion operator for short) if it  satisfies \sinti-\sintiv.
\end{definition}

There are other important postulates describing mainly the relationships between the results of merging a whole society and the results of merging its subsocieties. In order to establish such properties, let $j$, $k$ be a pair of agents in $\mathcal{S}$,  $N$ be any finite society of agents in $\mathcal{S}$, $\Set{N_1, N_2}$ be any partition of $N$, $\Phi$ be any $N$-profile,  $E_j$ be any $j$-profile, $E_k$ be any $k$-profile, and $E$ be any \ee\ in $\mathcal{E}$.

\begin{description}
\item[(ESF5)] If $E_j\neq E_k$, there exits  $E'$ in $\mathcal{E}$ such that $B\big(\nabla(E_j,E')\big)\not\equiv B\big(\nabla(E_k,E')\big)$
\item[(ESF6)]If $\Bigwedge_{i\in N}B(E_{i})\wedge B(E)\not\vdash\bot$, then $B\big(\nabla(\Phi,E)\big)\equiv\Bigwedge_{i\in N}B(E_{i})\wedge B(E)$
\item[(ESF7)] $B\big(\nabla(\Phi\upharpoonright_{_{N_1}},E)\big)\wedge B\big(\nabla(\Phi\upharpoonright_{_{N_2}},E)\big)\vdash B\big(\nabla(\Phi,E)\big)$
\item[(ESF8)] If $B\big(\nabla(\Phi\upharpoonright_{_{N_1}},E)\big)\wedge B\big(\nabla(\Phi\upharpoonright_{_{N_2}},E)\big)\not\vdash\bot$, then $B\big(\nabla(\Phi,E)\big)\vdash B\big(\nabla(\Phi\upharpoonright_{_{N_1}},E)\big)\wedge B\big(\nabla(\Phi\upharpoonright_{_{N_2}},E)\big)$
\end{description}

\textbf{(ESF5)} is a new postulate. It says  that given two different \ees, there is a   restriction $E$ that leads to different results at the level of beliefs, \ie\ the beliefs of the result of the operator applied to each \ee\ with the restriction  $E$, will not be equivalent.

\textbf{(ESF6)} expresses that if all the agents of the profile agree  at the level of beliefs with the restriction, this agreement will coincide with the belief resulting after application of the operator. This Postulate corresponds to Postulate (IC2) of IC merging operators \cite{KP02,KP11}.

\textbf{(ESF7)} tells us that for any partition of a group into two subgroups, the conjunction of belief of the result of applying the operator to each subgroup will be logically stronger than the beliefs resulting of applying the operator to the whole group. This Postulate corresponds to Postulate (IC5) of IC merging operators \cite{KP02,KP11}.

\textbf{(ESF8)} expresses that if we can divide a group into two subgroups such that the application of the operator to each subgroup leads to beliefs which  are mutually consistent, then the conjunction of these beliefs will be the beliefs resulting of applying the operator to the whole group.
This Postulate corresponds to Postulate (IC6) of IC merging operators \cite{KP02,KP11}.

\begin{definition}[Epistemic state fusion operators]
Let  $\nabla$ be an ES combination operator.  $\nabla$ is said to be an epistemic state  fusion operator (ES fusion operator for short)
if it satisfies the postulates \sinti-\sintviii.
\end{definition}

There are some important variants of these operators coming from the satisfaction of some special postulates. These postulates are
stated in what follows:

\begin{description}
\item[(ESF8W)] If $B\big(\nabla(\Phi\upharpoonright_{_{N_1}},E)\big)\wedge B\big(\nabla(\Phi\upharpoonright_{_{N_2}},E)\big)\not\vdash\bot$, then  $B\big(\nabla(\Phi,E)\big)\vdash B\big(\nabla(\Phi\upharpoonright_{_{N_1}},E)\big)\vee B\big(\nabla(\Phi\upharpoonright_{_{N_2}},E)\big)$
\end{description}

This property tells us that if a group is divided into two subgroups and, after application of the operator, the beliefs of the subgroups are consistent, then the beliefs of the whole group after application of the operator have to entail the disjunction of the beliefs of each subgroup.
This Postulate corresponds to Postulate (IC6') of IC merging operators \cite{KP02,KP11}.

Replacing the postulate \sintviii\ by the weaker postulate \sintviiiw, gives us the class of epistemic state  quasifusion operators.

\begin{definition}[Epistemic state quasifusion operators]
Let  $\nabla$ be an ES combination operator.  $\nabla$ is said to be an epistemic state  quasifusion operator (ES quasifusion operator for short)
if it satisfies the postulates \sinti--\sintvii\ and \sintviiiw.
\end{definition}

Note that if an operator satisfies \sintviii, necessarily it satisfies \sintviiiw. Thus, every ES fusion operator  is an ES  quasifusion operator. However, the converse is not true (see Section \ref{examples}).

\subsection{Faithful assignments}

An {\em assignment} is a function mapping epistemic state profiles into total preorders over interpretations. The intended meaning of these
mappings is encoding semantically, in some sense, the group preference.

\begin{definition}[Assignment and Basic Assignment]
      Let $(\mathcal{E},B,\mathcal{L}_{\mathcal P})$ be an epistemic space and $\mathcal{S}$ be a set of agents.  An {\em assignment} is a function \mbox{$f:\mathcal{P}(\mathcal{S,E})\rightarrow \mathds{P}(\mathcal{W_P})$} mapping each epistemic profile $\Phi$ into $f(\Phi)$, a total preorder over   $\mathcal{W_P}$.
      An assignment is called  {\em basic assignment} when, for any profiles $\Phi$ and $\Psi$, if $\Phi\equiv\Psi$ then $f(\Phi)= f(\Psi)$.

\end{definition}

Given an assignment $f$, we will denote by $\succeq_{\Phi}$, the total preorder $f(\Phi)$; and $\Phi\mapsto\succeq_\Phi$ will denote the mapping $f$. The  total preorder $\succeq_\Phi$ can be seen as the group plausibility preference over worlds:
  \begin{itemize}
    \item If $w\succeq_\Phi w'$, we will say that $w$ is {\em at least as plausible as} $w'$, for the agents group  in $\Phi$
    \item If $w\succ_\Phi w'$, we will say that $w$ is {\em more plausible than} $w'$, for the agents group  in $\Phi$
  \end{itemize}

In order to have  socially well behaved assignments, it is necessary to impose some rational properties. With the purpose of doing this, let $j$, $k$ be a pair of agents in $\mathcal{S}$,  $N$ be any finite society of agents in $\mathcal{S}$, $\Set{N_1, N_2}$ be any partition of $N$, $\Phi$ be any $N$-profile,  $E_j$ be any $j$-profile, $E_k$ be any $k$-profile, and $w$, $w'$ be any pair of interpretations in $\Val$.

\begin{enumerate}
\item[\assgi] If $E_j\neq E_k$, then $\succeq_{E_j}\neq\succeq_{E_k}$
\item[\assgii] If $\Bigwedge_{i\in N}B(E_{i})\not\vdash\bot$, then $\model{\Bigwedge_{i\in N} B(E_i)}=\max(\succeq_\Phi)$
\item[\assgiii] If $w\succeq_{{\Phi\upharpoonright_{N_1}}}w'$ and $w\succeq_{\Phi\upharpoonright_{N_2}}w'$ then $w\succeq_\Phi w'$
\item[\assgiv] If $w\succeq_{{\Phi\upharpoonright_{N_1}}}w'$ and $w\succ_{\Phi\upharpoonright_{N_2}}w'$, then $w\succ_\Phi w'$
\end{enumerate}

 Property \textbf{1} imposes  that different \ees\ lead to different total preorders (injectivity of the assignment restricted to profile of size one).

 Property \textbf{2}  tells us that, if there are models of the conjunction of the beliefs of the \ees\ of the profile, they are exactly the maximal models of the total preorder associated to the profile.

 Property \textbf{3} expresses that if one model $w_1$ is at least as plausible as   $w_2$ for one group, and the same occurs for a second group, then for the group resulting of the union of these groups, $w_1$ will be at least as plausible as $w_2$.

 Property \textbf{4} is similar to the previous one, except that if there is one preference strict for one of the subgroups, this will be the case for the whole group.

We think that the most entrenched preferences  should represent the beliefs. This is expressed more precisely by the following property:

\begin{definition}[Maximality Condition]
   Let $(\mathcal{E},B,\mathcal{L}_{\mathcal P})$ be an epistemic space and  $\mathcal{S}$ be a set of agents. The assignment
   $\Phi\mapsto\succeq_\Phi$ satisfies the maximality condition with respect to  $B$, if for all $i$ in $\mathcal S$ and all $i$-profile $E_i$, the following equation holds
\begin{equation*}\label{cond-min}
\model{B(E_i)}=\max(\succeq_{E_i})
\end{equation*}
\end{definition}

Note that any assignment $\Phi\mapsto\succeq_\Phi$ satisfying Property \assgii, with respect to $B$, satisfies the Maximality Condition.
Moreover, in presence of Properties  \assgiii\ and \assgiv, we have that Property \assgii\ is equivalent to the Maximality Condition.

\begin{proposition}\label{prop3}
   Let $(\mathcal{E},B,\mathcal{L}_{\mathcal P})$ be an epistemic space and  $\mathcal{S}$ be a set of agents. Suppose that the assignment $\Phi\mapsto\succeq_\Phi$ satisfies Properties \assgiii\ and \assgiv. Then $\Phi\mapsto\succeq_\Phi$ satisfies Property \assgii\ if, and only if, it satisfies the Maximality Condition.
\end{proposition}

The previous result shows some hidden connections between Properties \assgii, \assgiii\ and \assgiv\ under the Maximality Condition. Actually, under Maximality Condition,
Properties \assgiii\ and \assgiv\ entail Property \assgii. However the converse does not hold as we will observe in Section \ref{examples}.

The properties previously stated, together determine when an assignment has a good social behavior.

\begin{definition}[Faithful assignment]
 Let $(\mathcal{E},B,\mathcal{L}_{\mathcal P})$ be an epistemic space and $\mathcal{S}$ be a set of agents.  An assignment  $\Phi\mapsto\succeq_\Phi$ is called a faithful assignment, with respect to $B$, if it is  a basic assignment which satisfies Properties \assgi-\assgiv.
\end{definition}

The following property is weaker than Property \assgiv:

\begin{enumerate}
\item[\assgivp] If $w\succ_{{\Phi\upharpoonright_{N_1}}}w'$ and $w\succ_{\Phi\upharpoonright_{N_2}}w'$ then $w\succ_\Phi w'$
\end{enumerate}

This expresses that given two  alternatives, $w,\;w'$, if for one of the groups,  $w$ is more plausible than $w'$, and the same relation occurs for other group, then, putting the groups together, $w$ is still more plausible than  $w'$.

Note that Property \assgiv\ entails Property \assgivp. However, the converse is not true as we will see in Section \ref{examples}.

Replacing Property \assgiv\ by Property \assgivp\ in the definition of faithful assignments gives us another important class of assignments:

\begin{definition}[Quasifaithful assignment]
Let $(\mathcal{E},B,\mathcal{L}_{\mathcal P})$ be an epistemic space and $\mathcal{S}$ be a set of agents.  An assignment  $\Phi\mapsto\succeq_\Phi$ is called a quasifaithful assignment, with respect to $B$, if it is a basic assignment in which  Properties \assgi-\assgiii\ and \assgivp\ hold.
\end{definition}

\subsection{Representation theorems}
We present some results which help to understand the behavior of the operators defined previously. This allows us to describe the operators, at least partially, in a semantical way.
Thus, we call these results weak representation
 by  opposition to some  results of representation in concrete structures like in \cite{MP11}, where we can effectively build the operator starting from an assignment.

\begin{theorem}[Weak representation for ES basic fusion operators]\label{teo1}

An ES combination operator $\nabla$ is an ES basic fusion operator \miff there exists
a unique basic assignment  $\Phi\mapsto\succeq_\Phi$ such that:
\def\theequation{B-Rep}
\begin{equation}\label{B-Rep}
\model{B\big(\nabla(\Phi,E)\big)}=max\big(\model{B(E)},\succeq_\Phi\big)
\end{equation}
\end{theorem}
\setcounter{equation}{2}

 It is worth to note that we can obtain a variant of this result with a weak version of \sintii\ where the equivalence between profiles is not necessary, just like the version presented in \cite{MP11}. More precisely, the operators satisfying this weak version of \sintii\ plus \sinti, \sintiii and \sintiv, are exactly the operators  weakly represented  by a simple assignment.

The Weak representation theorem allows us to obtain  very tight relations between the (syntactical) postulates  for fusion and the properties of the assignments.
More precisely, we have the following result:

\begin{proposition}\label{prop4}
Let  $\nabla$ be an ES basic fusion operator and $\Phi\mapsto\succeq_\Phi$ be the basic assignment associated to $\nabla$ by Theorem \ref{teo1}. Then, the following conditions hold:
\begin{enumerate}
\item[(i)] $\nabla$ satisfies \sintv\  iff $\Phi\mapsto\succeq_\Phi$ satisfies \assgi.
\item[(ii)] $\nabla$ satisfies \sintvi\ iff $\Phi\mapsto\succeq_\Phi$ satisfies  Properties \assgii.
\item[(iii)] $\nabla$ satisfies \sintvii\ iff $\Phi\mapsto\succeq_\Phi$ satisfies \assgiii.
\item[(iv)] $\nabla$ satisfies \sintviii\ iff $\Phi\mapsto\succeq_\Phi$ satisfies \assgiv.
\item[(v)] $\nabla$ satisfies \sintviiiw\ iff $\Phi\mapsto\succeq_\Phi$ satisfies \assgivp.
\end{enumerate}
\end{proposition}

The next two theorems are important representation results. They are straightforward consequences  of Proposition \ref{prop4}.

\begin{theorem}[Weak representation for ES  fusion operators]\label{teo2}
An ES combination operator $\nabla$ is an ES  fusion operator \miff there exists
a unique faithful assignment  $\Phi\mapsto\succeq_\Phi$ satisfying \eqref{B-Rep}.
\end{theorem}

\begin{theorem}[Weak representation for ES  quasifusion operators]\label{teo3}
An ES combination operator $\nabla$ is an ES  quasifusion operator \miff there exists
a unique quasifaithful assignment  $\Phi\mapsto\succeq_\Phi$ satisfying \eqref{B-Rep}.
\end{theorem}


{
 Unlike the representation theorem for IC merging operators \cite{KP02,KP11}, the previous theorems do not allow the construction of $\nabla$ from the basic assignment. However, they allow the representation of
 the  beliefs of the result via the total preorders of the assignment.
}

\section{Impossibility in Epistemic Fusion}\label{Impossibility}

We would like the fusion process to have  properties guaranteeing more global satisfaction. These good properties are based in the following principles:

 \begin{enumerate}
   \item[a.] Any result is possible in the framework of the restrictions of the system.
   \item[b.] If all the agents have the same epistemic state it is this state which determines the result of the fusion.
   \item[c.] The fusion process depends only on how  the restrictions are related in the individual epistemic states.
   \item[d.] The group belief base, obtained as the result of the fusion process, does not depend on one unique agent.
 \end{enumerate}

 Unfortunately, as we will see in this section, all these principles cannot coexist. More precisely, we will prove that for a fusion process  satisfying the three first principles, necessarily there exists an agent that will impose his beliefs.

 Next, we formulate the above principles in  logical terms based on our operators over epistemic states.

\subsection{Other social properties for fusion operators}

In order to give a formulation of the  principles stated above, we will suppose from now on that there are at least two propositional variables, \ie\ the set $\mathcal{P}$ of atomic formulas has cardinality greater or equal to two. In particular, there will be at least four interpretations in $\mathcal{W_P}$.

The first property we present will be called {\em Standard Domain}.
 It states some ``richness'' in the set of results of the fusion process. This is an approximation of the first principle which says that any  result
 is possible in the framework of the restrictions of the system.

\begin{description}
\item[{\bf(ESF-SD)}] For any $i$ in $\mathcal S$, for any triple of interpretations $w$, $w'$ and $w''$ in $\mathcal{W_P}$, and any couple of epistemic states  $E_{w,w'}$ and $E_{w',w''}$, such that
    \mbox{$\model{B(E_{w,w'})}=\{w,w'\}$} and $\model{B(E_{w',w''})}=\{w',w''\}$, the following conditions hold:
        \begin{enumerate}
          \item[(i)]  There exists an $i$-profile  $E_i$  such that $B\big(\nabla(E_i,E_{w,w'})\big)\equiv\varphi_{w,w'}$ and  $B\big(\nabla(E_i,E_{w',w''})\big)\equiv \varphi_{w',w''}$
          \item[(ii)] There exists an $i$-profile  $E_i$  such that $B\big(\nabla(E_i,E_{w,w'})\big)\equiv\varphi_{w,w'}$ and $B\big(\nabla(E_i,E_{w',w''})\big)\equiv \varphi_{w'}$
          \item[(iii)] There exists an $i$-profile  $E_i$  such that $B\big(\nabla(E_i,E_{w,w'})\big)\equiv\varphi_{w}$ and $B\big(\nabla(E_i,E_{w',w''})\big)\equiv\varphi_{w',w''}$
          \item[(iv)] There exists an $i$-profile  $E_i$  such that $B\big(\nabla(E_i,E_{w,w'})\big)\equiv\varphi_{w}$ and $B\big(\nabla(E_i,E_{w',w''})\big)\equiv\varphi_{w'}$
        \end{enumerate}
\end{description}


\sintSD\ establishes that any result is possible  when the constraints have beliefs with at most two models. Actually, this postulate is more related to the non-imposition postulate in Social Choice Theory.

\begin{observation}\label{remark-shapes}
It is worth noting that for basic fusion operators (those satisfying Theorem~\ref{teo1}), the satisfaction of this postulate is equivalent to the following fact:
 for any agent $i$ in $\mathcal{S}$, any triple of interpretations $w$, $w'$ and $w''$ in $\mathcal{W_P}$, and any total preorder $\succeq$ between these interpretations, there is an $i$-profile $E_i$ such that $\succeq=\succeq_{E_i}\upharpoonright_{\Set{w,w',w''}}$.
\end{observation}

The previous Observation and some natural combinatorial arguments tell us that  if the epistemic space is reduced to the consistent formulas modulo logical equivalence and the function $B$ is the identity then the basic fusion operators satisfying the Maximality Condition can not satisfy \sintSD.
That is precisely the next result.

\begin{theorem}\label{new-theorem-formulas}

Consider the epistemic space $(\mathcal{E},B,\mathcal{L_P})$ where $\mathcal{E}=\mathcal{L_P}^\ast/\!\equiv$,
with $\mathcal{L_P}$  a set of formulas built  over two propositional variables and
$B(\varphi)=\varphi$ for every element\footnote{By abuse we write $\varphi$ for the equivalence class of the propositional formula $\varphi$.} $\varphi$ in $\mathcal{E}$.
Let $\nabla:\mathcal{P}(\mathcal{S,E})\times \mathcal{E}\dans \mathcal{E}$ be an ES basic fusion operator satisfying the Maximality Condition.
Then $\nabla$ does not satisfy \sintSD.
\end{theorem}

The next result shows the richness and versatility of the fusion operators which satisfy this property:

\begin{proposition}\label{prop5}
 Let $(\mathcal{E},B,\mathcal{L_P})$ be an epistemic space, $\mathcal{S}$ be a set of agents and  $\nabla$  an ES basic fusion operator.  If  $\nabla$ satisfies \sintSD,  then  for
   any agent $i$ in $\mathcal{S}$, any $i$-profile $E_i$, any different triple of interpretations $w$, $w'$, $w''$ in $\mathcal{W_P}$ and any epistemic states $E_{w,w'}$, $E_{w,w''}$, $E_{w',w''}$   in $\mathcal{E}$, such that $\model{B(E_{w,w'})}=\{w,w'\}$, $\model{B(E_{w,w''})}=\{w,w''\}$ and $\model{B(E_{w',w''})}=\{w',w''\}$, the following conditions hold:
\begin{enumerate}
  \item[(i)] For all $j$ in $\mathcal{S}$, there exists a $j$-profile $E_j$ such that
  \begin{itemize}
    \item $B\big(\nabla(E_j,E_{w,w'})\big)\equiv B\big(\nabla(E_j,E_{w,w''})\big)\equiv\varphi_{w}$, and
    \item $B\big(\nabla(E_j,E_{w',w''})\big)\equiv B\big(\nabla(E_i,E_{w',w''})\big)$.
  \end{itemize}

  \item[(ii)] For all $k$ in $\mathcal{S}$, there exists a $k$-profile $E_k$ such that
  \begin{itemize}
    \item $B\big(\nabla(E_k,E_{w,w'})\big)\equiv\varphi_{w}$,
    \item $B\big(\nabla(E_k,E_{w',w''})\big)\equiv\varphi_{w''}$, and
    \item $B\big(\nabla(E_k,E_{w,w''})\big)\equiv B\big(\nabla(E_i,E_{w,w''})\big)$.
  \end{itemize}
\end{enumerate}

\end{proposition}

The following two properties try to capture the meaning of the second of the principles above mentioned. The first of them is given in terms of similarity between the epistemic states of the agents. Actually, it is natural to think that if the agents totally coincide \ie\ all have the same epistemic state, then the result of the group can be determined for any individual agent. More precisely, we have the following postulate called {\em Unanimity condition}:

\begin{description}
\item[(ESF-U)]\em For all $N$ in $\mathcal{F}^\ast(\mathcal{S})$, for each  $N$-profile $\Phi$ and for every epistemic state $E$  in $\mathcal{E}$, if $E_i=Ej$, for any pair  $i$, $j$ in $N$, then, for all $i$ in $N$, $B\big(\nabla(\Phi,E)\big)\equiv B\big(\nabla(E_i,E)\big)$.
\end{description}

\sintU\ establishes that  if all the agents in the society have the same epistemic state,  then the result of the fusion is exactly the fusion (or the revision\footnote{When the society is formed by one unique agent, it is well known that the process of fusion corresponds to a revision of the epistemic state of the agent by the integrity constraints (see for instant \cite{KP11}).}) of the  society conformed by any of the agents given the integrity constraints.

The following result is important. It establishes that the  ES fusion operators satisfy the Unanimity condition.

\begin{proposition}\label{prop6}
       Let $\nabla$ be an ES combination operator. If $\nabla$ satisfies \sintii, \sintvii\ and \sintviii, it satisfies \sintU.
\end{proposition}

The converse of this result does not hold, as can be seen in Section \ref{examples}.

In the following result we give a semantic characterization of the Unanimity condition.

\begin{proposition}\label{prop7}
  Let $\nabla$ be an ES basic fusion operator. Then the following statements are equivalent:
\begin{enumerate}
  \item[(i)] $\nabla$ satisfies \sintU.
  \item[(ii)] For all $N$ in $\mathcal{F}^\ast(\mathcal{S})$, for each $N$-profile $\Phi$, for every epistemic state $E$ in $\mathcal{E}$ such that $\Big|\model{B(E)}\Big|\leq 2$, if $E_i=E_j$ for each pair $i$, $j$ in $N$, then $B\big(\nabla(\Phi,E)\big)\equiv B\big(\nabla(E_i,E)\big)$, for all $i$ in $N$.
  \item[(iii)] The assignment $\Phi\mapsto\succeq_\Phi$, representing  $\nabla$, satisfies the following property:\smallskip
  \begin{description}
    \item[({\em u})] For each society $N$, for each $N$-profile $\Phi$,  if $E_i=E_j$ for each couple $i$, $j$ in $N$, then  $\succeq_\Phi=\succeq_{E_i}$ for all $i$ in $N$.
  \end{description}
\end{enumerate}
\end{proposition}

Note that this last condition is  closer to the formulations of Unanimity in Social Choice Theory.

Now we give a second syntactical formulation of the second principle: if all the agents reject a given alternative, this alternative will be rejected
in the result of fusion. This form of excluding unanimity is known in Social Choice Theory like {\em Pareto condition}. We keep this name for the following postulate:

\begin{description}
\item[(ESF-P)]\em For all $N$ in $\mathcal{F}^\ast(\mathcal{S})$, for each $N$-profile $\Phi$, for all epistemic states $E$ and $E'$ in $\mathcal{E}$,
if $\Bigwedge_{i\in N} B\big(\nabla(E_i,E)\big)\not\vdash\bot$ and
 $B\big(\nabla(E_i,E)\big)\wedge B(E')\vdash\bot$  for all $i$ in $N$, then
    $B\big(\nabla(\Phi,E)\big)\wedge B(E')\vdash\bot$.
\end{description}

The following result entails that ES fusion operators satisfy the Pareto condition:

\begin{proposition}\label{prop16}
      If $\nabla$ is an ES combination operator that satisfies \sintvii\ and \sintviiiw, then $\nabla$ also satisfies \sintP.
      \end{proposition}

The converse is not true as  will be seen in Section \ref{examples}.

The Pareto condition  has also a semantical characterization that we give in the next result:

\begin{proposition}\label{prop8}
Let $\nabla$ be an ES basic fusion operator. Then the following statements are equivalent:
\begin{enumerate}
  \item[(i)] $\nabla$ satisfies \sintP
  \item[(ii)] For all $N$ in $\mathcal{F}^\ast(\mathcal{S})$, for every $N$-profile $\Phi$, for all epistemic states $E$ and $E'$ in $\mathcal{E}$ such that $\Big|\model{B(E)}\Big|\leq 2$, if $B\big(\nabla(E_i,E)\big)\wedge B(E')\vdash\bot$ for all $i$ in $N$, then $B\big(\nabla(\Phi,E)\big)\wedge       B(E')\vdash\bot$
  \item[(iii)] The assignment $\Phi\mapsto\succeq_\Phi$, representing  $\nabla$, satisfies the following property:\smallskip
  \begin{description}
    \item[({\em p})] For all $N$ in $\mathcal{F}^\ast(\mathcal{S})$, for each $N$-profile $\Phi$,  for all interpretations $w$ y $w'$,
    if  $w\succ_{E_i} w'$ for all $i$ in $N$, then $w\succ_\Phi w'$
  \end{description}
\end{enumerate}
\end{proposition}

We continue stating the syntactical postulate aiming to catch  the third principle:
{  \em the fusion process depends only on how  the restrictions in the individual epistemic states are related.
}

In Social Choice Theory this idea is captured by  the postulate known as {\em Independence of Irrelevant Alternatives}. In our framework this postulate
will be stated in the following manner:

\begin{description}
\item[(ESF-I)]\em For all $N$ in $\mathcal{F}^\ast(\mathcal{S})$, for all  $N$-profiles $\Phi$ and $\Phi'$, for every epistemic state
$E$ we have $B\big(\nabla(\Phi,E)\big)\equiv B\big(\nabla(\Phi',E)\big)$, whenever for each epistemic state $E'$ such that
      \mbox{$B(E')\vdash B(E)$}, we have
    $B\big(\nabla(E_j,E')\big)\equiv B\big(\nabla(E'_j,E')\big)$ for all $j\in N$.
\end{description}

This property, called {\em Independence condition}, essentially  says the following: Given an integrity constraint, if each agent in the fusion process has  two possible choices of epistemic states, and  if revising these epistemic states by  integrity constraints having beliefs stronger than the given  integrity constraint,  the beliefs of the epistemic states resulting coincide, then
the result of the fusion of the society of agents under the given integrity constraint is the same, at the level of beliefs, for any  choice of the epistemic state made by each agent.

The next result gives a simplification of the Postulate \sintI\ for the ES basic fusion operators and also a semantic characterization.

\begin{proposition}\label{prop9}
Let $\nabla$ be an ES basic fusion operator. Then the following statements are equivalent:
\begin{enumerate}
  \item[(i)] $\nabla$ satisfies \sintI
  \item[(ii)] For all $N$ in $\mathcal{F}^\ast(\mathcal{S})$, for all  $N$-profiles $\Phi$ and $\Phi'$, for all
    $E$ in $\mathcal{E}$ such that \mbox{$\Big|\model{B(E)}\Big|\leq 2$}, if\linebreak
   $B\big(\nabla(E_j,E)\big)\equiv B\big(\nabla(E_j',E)\big)$, for all $j$ in $N$, then
      $B\big(\nabla(\Phi,E)\big)\equiv B\big(\nabla(\Phi',E)\big)$
  \item[(iii)] The assignment $\Phi\mapsto\succeq_\Phi$, representing  $\nabla$, satisfies the following property:\smallskip
  \begin{description}
    \item[({\em ind})] For all $N$ in $\mathcal{F}^\ast(\mathcal{S})$, for all  $N$-profiles $\Phi$ and $\Phi'$, for all interpretations $w$ and $w'$, if
        $\succeq_{E_i}\upharpoonright_{\{w,w'\}}=\succeq_{E_i'}\upharpoonright_{\{w,w'\}}$, for all $i$ in $N$, then $\succeq_{\Phi}\upharpoonright_{\{w,w'\}}=\succeq_{\Phi'}\upharpoonright_{\{w,w'\}}$
  \end{description}
\end{enumerate}
\end{proposition}

It is worth noting that in presence of \sintI\ and the basic postulates, Postulate \sintU\ entails \sintP, \ie\ unanimity is stronger than Pareto condition. That is the following proposition:

    \begin{proposition}\label{prop10}
     If an ES basic fusion operator $\nabla$ satisfies \sintU\ and \sintI, then it also satisfies \sintP.
      \end{proposition}

      The converse is not true, that is, even in presence of \sintI, Postulate \sintP\ does not entail \sintU. Actually, it is not difficult to build operators satisfying \sintP, \sintI\ but for which \sintU\ does not hold (see Section \ref{examples}). Moreover, as we can see in Section \ref{examples}, \sintP\ does not entail \sintI, even if \sintU\ holds.

Next we state the postulate related with the fourth principle: {\em the group belief base, obtained as the result of the fusion process, does not depend on one unique agent.} Actually, we establish the negative form which says that there is an agent that imposes his will, a {\em dictatorial agent}. This is the postulate that the {\em good} operators should avoid.

    \begin{description}
        \item[(ESF-D)]\em For all $N$ in $\mathcal{F}^\ast(\mathcal{S})$  there exists an agent $d_N$ in $N$ such that,  for all $N$-profile $\Phi$ and  for all epistemic state $E$ in $\mathcal{E}$,
            $B\big(\nabla(\Phi,E)\big)\vdash B\big(\nabla(E_{d_N},E)\big)$.
    \end{description}

The operators satisfying the previous postulate are called {\em Dictatorial operators}. Suppose that the operator $\nabla$ is dictatorial,
then given $N$ in  $\mathcal{F}^\ast(\mathcal{S})$, an agent $d_N$ satisfying
the property in \sintD, is called a {\em dictador in $N$} or simply {\em $N$-dictator}, with respect to $\nabla$.

The following result is a semantic characterization of Dictatorial operators:

\begin{proposition}\label{prop11}
Let $\nabla$ be an ES basic fusion operator. Then the following statements are equivalent:
\begin{enumerate}
  \item[(i)] $\nabla$ is dictatorial.
  \item[(ii)] For all $N$ in $\mathcal{F}^\ast(\mathcal{S})$, there exists $d_N$ in $N$ such that, for every $N$-profile $\Phi$ and for all epistemic state $E$ in $\mathcal{E}$, if  $\Big|\model{B(E)}\Big|\leq 2$ then
      $B\big(\nabla(\Phi,E)\big)\vdash B\big(\nabla(E_{d_N},E)\big)$.
  \item[(iii)] The assignment $\Phi\mapsto\succeq_\Phi$, representing  $\nabla$, satisfies the following property:\smallskip
  \begin{description}
    \item[({\em d})] For all $N$ in $\mathcal{F}^\ast(\mathcal{S})$, there exists $d_N$ in $N$ such that for all interpretations $w$ and $w'$, if $w\succ_{E_{d_N}} w'$, then $w\succ_\Phi w'$.
  \end{description}
\end{enumerate}
\end{proposition}

The next result establishes that Dictatorial operators satisfy the Pareto Condition.

\begin{proposition}\label{prop12}
  Let $\nabla$ be an ES combination operator. If $\nabla$ satisfies \sintD, then $\nabla$ satisfies \sintP.
\end{proposition}

However, there are dictatorial operators for which \sintU\ does not hold (see Section \ref{examples}).

\subsection{Coalitions and decisional power}
In many situations a group of agents or individuals who integrate a society form alliances or coalitions in order to achieve  common goals.  Thus, in a first step to define important types of alliances, we will say simply that  a {\em coalition} is determined by a group of individuals of a society. More precisely, we state:

 \begin{definition}[Coalitions]
    Given a finite society of agents $N$ in $\mathcal{F}^\ast(\mathcal{S})$, a coalition of agents in $N$, or simply an $N$-coalition, is a set $D$ of agents in $N$.
 \end{definition}

Some coalitions in logic-based fusion have an important decisional power. This kind of decisional power   can be illustrated as follows:
Let us suppose that the agents of a society have to decide about the acceptation of certain information under some given condition (integrity constraints). Also suppose that, under such restriction, the new information does not satisfy  some of the agents (the coalition) at all, and the rest of the agents accept this information. When the global decision is to reject  the new information, it is clear that the coalition has imposed its view. This kind of coalitions, with great decisional power, are called {\em locally decisive coalitions}.

\begin{definition}[Locally decisive coalitions]
 Let $N$ be in $\mathcal{F}^\ast(\mathcal{S})$, $\nabla$ be an ES combination operator, $E$, $E'$ be a pair of \ees\ and $D$ be  an $N$-coalition. We say that $D$ is a locally decisive $N$-coalition for $E$ against $E'$, with respect to $\nabla$, denoted by ${E}D^\nabla{E'}$, if for each $N$-profile $\Phi$ satisfying:
\begin{enumerate}
  \item[(i)] $B\big(\nabla(E_i,E)\big)\wedge B(E')\vdash\bot$, for all $i$ in $D$;
  \item[(ii)] $B\big(\nabla(E_j,E)\big)\equiv B(E')$, for all $j$ in $N\setminus D$; and
  \item[(iii)] $\Bigwedge_{i\in D} B\big(\nabla(E_i,E)\big)\not\vdash\bot$
\end{enumerate}
 we have that $B\big(\nabla(\Phi,E)\big)\wedge B(E')\vdash\bot$.
\end{definition}

It is also possible that the rejection of the new information will be achieved by a coalition independently from the opinion of the rest of the agents involved in the merging process. This kind of coalitions are called {\em decisive coalitions}.

\begin{definition}[Decisive coalition]\label{decisive}
Let $N$ be in $\mathcal{F}^\ast(\mathcal{S})$, $\nabla$ be an ES combination operator, $E$, $E'$ be a pair of \ees\ and $D$ be  an $N$-coalition. We say that $D$ is a decisive $N$-coalition for $E$ against $E'$, with respect to $\nabla$, denoted by ${E}D^{\nabla^\ast}{E'}$, if for every $N$-profile $\Phi$ satisfying:
\begin{enumerate}
  \item[(i)] $B\big(\nabla(E_i,E)\big)\wedge B(E')\vdash\bot$, for all $i$ in $D$; and
  \item[(ii)] $\Bigwedge_{i\in D} B\big(\nabla(E_i,E)\big)\not\vdash\bot$
\end{enumerate}
we have that $B\big(\nabla(\Phi,E)\big)\wedge B(E')\vdash\bot$.

We will say that $D$ is a decisive $N$-coalition, with respect to $\nabla$, if  ${E}D^{\nabla^\ast}{E'}$ for any pair $E$, $E'$ of \ees.
\end{definition}

\begin{observation}\label{obs1}\rm For any finite society $N$ in $\mathcal{F}^\ast(\mathcal{S})$,  any ES combination operator $\nabla$,  any $N$-coalition $D$ and any tuple of \ees\ $E$, $E'$, $\overline{E}$, $\overline{E'}$ in $\mathcal{E}$, it is easy to see that the following conditions hold:

\begin{enumerate}
 \item[a.] If ${E}D^{\nabla^\ast}{E'}$, then ${E}D^\nabla{E'}$

 \item[b.] If $\nabla$ satisfies \sinti\ and $B(E)\equiv B(E')$, or  $B(E)$ is inconsistent with $B(E')$, then $ED^{\nabla^\ast}{E'}$ and therefore $ED^\nabla{E'}$

 \item[c.] If $\nabla$ satisfies \sintii, $B(E)\equiv B(\overline{E})$ and
      $B(E')\equiv B(\overline{E'})$, then ${E}D^{\nabla^\ast}{E'}$ if, and only if, ${\overline{E}}D^{\nabla^\ast}{\overline{E'}}$

    \item[d.] If $\nabla$ satisfies \sinti\ and $D$ es decisive, then $D$ is nonempty\footnote{If $D=\emptyset$ the two premisses in Definition \ref{decisive} are always true. Thus, if $D$ were  decisive, taking $E=E'$, we get $B\big(\nabla(\Phi,E)\big)\wedge B(E)\vdash\bot$,
     a contradiction.}.

    \item[e.] If $\nabla$  satisfies \sintP, then $N$ is a decisive $N$-coalition by itself.

\end{enumerate}
\end{observation}

The following result is important:

\begin{proposition}\label{propnew-dic}
  Let $\nabla$ be a ES combination operator. Let $N$ be a finite society of agents.   If $D$ is a decisive $N$-coalition
  and its cardinality is one, then the only element of $D$ is a dictator in $N$.
\end{proposition}

 Next, we introduce the first result about decisive coalition. It establishes some necessary and sufficient conditions that must be satisfied by a basic assignment, associated  to certain ES basic fusion operator, in order for it to admit a decisive coalition. Due to the similarity between the Pareto condition and the decisive coalition notion, the proof of this result is analogue to that of Proposition \ref{prop8}.

\begin{proposition}\label{prop13}
   Given $N$ in $\mathcal{F}^\ast(\mathcal{S})$, $\nabla$ an ES basic fusion operator and $\Phi\mapsto\succeq_\Phi$ the basic assignment  associated to $\nabla$, the following claims about an $N$-coalition $D$ are equivalent:
    \begin{enumerate}
        \item[(i)] $D$ is a decisive $N$-coalition, with respect to $\nabla$.
        \item[(ii)] For each $N$-profile $\Phi$ and any pair of epistemic states  $E$, $E'$ in $\mathcal{E}$, if $\Big|\model{B(E)}\Big|\leq 2$ and $B\big(\nabla(E_i,E)\big)\wedge B(E')\vdash\bot$, for all $i$ in $D$, then $B\big(\nabla(\Phi,E)\big)\wedge B(E')\vdash\bot$.
        \item[(iii)] For each $N$-profile $\Phi$, and all interpretations $w$, $w'$, if $w\succ_{_{E_i}} w'$, for all $i$ in $D$, then $w\succ_\Phi w'$.
    \end{enumerate}

\end{proposition}

Let us note that the last result, in addition to giving a semantic characterization of operators that admit decisive coalitions (point (iii)),  establishes that in order to determine  the decisive nature of a coalition with respect to an ES basic operator, it is enough  considering epistemic states whose most entrenched beliefs have at most two models (point (ii)).

Under certain rational properties over a merging process, if a coalition is locally decisive between a pair of epistemic states of a certain kind, then we can change slightly the pair of epistemic states and the coalition will continue to be locally decisive for this new pair of epistemic states. More precisely, we have the following two results:

\begin{proposition}[First propagation lemma]\label{prop14}
   Given $N$ in $\mathcal{F}^\ast(\mathcal{S})$, an $N$-coalition $D$, an ES basic fusion operator $\nabla$  satisfying \sintSD, \sintP\ and \sintI, and  $w$, $w'$ a pair of different interpretations. If $E_{w,w'}$, $E_{w'}$ are epistemic states such that $\model{B(E_{w,w'})}=\{w,w'\}$, $\model{B(E_{w'})}=\{w'\}$ and ${E_{w,w'}}D^\nabla{E_{w'}}$, then, for any interpretation $w''$, with \linebreak$w''\not\models B(E_{w,w'})$ and any pair of epistemic states $E_{w,w''}$, $E_{w''}$, with $\model{B(E_{w,w''})}=\{w,w''\}$ and
   $\model{B(E_{w''})}=\{w''\}$, we have ${E_{w,w''}}D^{\nabla^\ast}{E_{w''}}$.
\end{proposition}

\begin{proposition}[Second propagation lemma]\label{prop15}
   Given $N$ in $\mathcal{F}^\ast(\mathcal{S})$, an $N$-coalition $D$, an ES basic fusion operator $\nabla$ satisfying \sintSD, \sintP\ and \sintI, and  $w$, $w'$ a pair of different interpretations. If $E_{w,w'}$, $E_{w'}$ are epistemic states such that $\model{B(E_{w,w'})}=\{w,w'\}$, $\model{B(E_{w'})}=\{w'\}$ and ${E_{w,w'}}D^\nabla{E_{w'}}$, then, for any interpretation $w''$ such that $w''\not\models B(E_{w,w'})$ and any pair of epistemic states $E_{w',w''}$ and $E_{w''}$ such that  $\model{B(E_{w,w''})}=\{w,w''\}$ and    $\model{B(E_{w''})}=\{w''\}$, we have ${E_{w',w''}}D^{\nabla^\ast}{E_{w'}}$.
\end{proposition}

The previous two propagation results  will allow us to show that if a coalition is locally decisive with respect to a given pair of epistemic states  then such coalition is decisive. More precisely we have the following theorem:

\begin{theorem}[Generalized propagation theorem]\label{teo5}
   Consider $N$ in $\mathcal{F}^\ast(\mathcal{S})$, let $\nabla$ be an ES basic fusion operator satisfying \sintSD, \sintP\ and \sintI, and $D$ be an $N$-coalition. If there exist different interpretations $w$ and $w'$, and \ees\ $E_{w,w'}$ and $E_{w'}$, such that $\model{B(E_{w,w'})}=\{w,w'\}$, $\model{B(E_{w'})}=\{w'\}$ and ${E_{w,w'}}D^\nabla{E_{w'}}$,  then $D$ is a decisive coalition.
\end{theorem}

\subsection{Impossibility results}
There exist some links  between Weak Standard Domain, Unanimity, Pareto, Independence conditions and the fact that a merging process admits a dictatorial agent. Such links are revealed through the Main Impossibility Theorem below. This theorem shows that if Weak Standard Domain, Pareto and  Independence conditions hold, an ES basic fusion operator is really a dictatorial operator.

\begin{theorem}[Main impossibility theorem]\label{teo6}
Let $\nabla$ be an ES basic fusion operator. If $\nabla$ satisfies \sintSD, \sintP\ and \sintI, then \sintD\ holds.
\end{theorem}

Note that the condition of Transitive explanations in Social Choice Theory is actually implicit in the representation of ES basic fusion operators.

 Postulates \sintSD, \sintP\ and \sintI\ are only sufficient conditions  in order to have an ES basic fusion operator be dictatorial. Actually, as we will see in Section \ref{examples}, there exist ES basic fusion operators that satisfy \linebreak\sintD\ but do not satisfy  \sintI, namely the {\em $\Sigma$-pseudoprojective ES fusion operators}. There are also dictatorial operators which do not satisfy \sintSD, namely the {\em linearized projective operators}
 (see Section \ref{examples}).

By the connection between the postulates established in Section \ref{operators-def} and the Pareto condition which has been revealed in Proposition  \ref{prop16} and the Main Impossibility Theorem  we have:

\begin{corollary}\label{cor1}
  Any ES basic fusion operator  that satisfies \sintSD, \sintvii, \sintviiiw\ (or \sintviii) and \sintI, also satisfies \sintD.
\end{corollary}

Similarly, by virtue of the relationship between the Unanimity and Pareto conditions, from Proposition \ref{prop10} and the Main Impossibility Theorem we have the following corollary:

\begin{corollary}\label{cor2}
  Any ES basic fusion operator  that satisfies \sintSD, \sintU\ and \sintI, also satisfies \sintD.
\end{corollary}

\section{Concrete examples of ES basic fusion operators}\label{examples}
{
In this section we assume that the epistemic states are total preorders over interpretations and we have\linebreak $\model{B(\succeq)}=\max(\succeq)$ for any epistemic state $\succeq$. Then, with this concrete representation of epistemic states,
we will define some ES basic fusion operators.

\subsection{Some classes of ES basic fusion operators}

First, we are going to define a class of ES basic fusion operators using aggregation functions. These operators are called {\em aggregation-based ES fusion operators}. In order to do this, let us recall the definition of aggregation functions.

\begin{definition}[Aggregation functions]
An (symmetric) aggregation function $F$ is a total function associating a
nonnegative real number to every finite tuple of nonnegative integers such that
for any $x_1, . . . , x_n, x, y$ in $\mathds{Z}^+$:

\begin{description}
    \item[Monotony] $F(x_1,\dots,x,\dots,x_n)\geq F(x_1,\dots,y,\dots,x_n)$, if $x\geq y$
    \item[Minimality]  $F(x_1,x_2,\dots,x_n)=0$ if and only if $x_1=x_2=\cdots=x_n=0$
    \item[Identity] $F(x)=x$
    \item[Symmetry] For any permutation $\sigma$, $F\big(x_1,x_2,\dots,x_n\big)=F\big(\sigma(x_{1},x_{2},\dots, x_{n})\big)$
\end{description}
\end{definition}

Two examples of aggregation functions are {\em sum} and {\em max}. They are classic examples in the study of logic-based fusion \cite{KP02}.

\begin{description}
  \item[sum] $\sum(x_1,x_2,\cdots,x_n)=\sum x_i$
  \item[max] $\max(x_1,x_2,\dots,x_n)=\max\{x_1,x_2,\dots,x_n\}$
\end{description}

An aggregation function $F$ induces a total preorder over $\mathcal{W_P}$, for any profile of epistemic states as follows:

  {\em For any $N=\{i_1,i_2,\dots,i_n\}$ in $\mathcal{F}^\ast(\mathcal{S})$,  any $N$-profile $\Phi=(\succeq_{i_1},\succeq_{i_2},\dots,\succeq_{i_n})$, and any couple of interpretations $w,w'$ in $\mathcal{W_P}$
\begin{equation}
w\succeq_\Phi^F w'\mbox{ iff }\; F(r_{i_1}(w),\dots,r_{i_n}(w))\geq F(r_{i_1}(w'),\dots,r_{i_n}(w'))
\label{eq:01}
\end{equation}
where $r_i(x)$ is the rank of an interpretation  $x$ in the total preorder $\succeq_i$, for each agent $i$ in $N$.
}

It is easy to see that, for a given profile of epistemic states $\Phi$, the relation $\succeq_\Phi^F$ is a total preorder over $\mathcal{W_P}$. Thus, for any aggregation function $F$ it is possible to define an assignment $\Phi\mapsto\succeq_\Phi^F$ which {\em preserves the structure of epistemic states}, that is, for all $i$ in $\mathcal{S}$, $\succeq_{\succeq_i}=\succeq_i$ for any total $i$-profile (epistemic state) $\succeq_i$. When an assignment preserves the structure of epistemic states we also say that the assignment is {\em structure preserving}. Note that these assignments trivially are basic assignments that satisfy  Property \assgi.

From the definition \eqref{eq:01}, in addition to the equations \eqref{eq:0} and \eqref{B-Rep}, it is possible to build an ES basic fusion operator $\nabla^F$ associated to an aggregation function $F$ as follows:
\begin{equation}
\nabla^F(\Phi,\succeq)=\succeq^{\mathrm{lex}(\succeq,\succeq_\Phi^F)}
\label{eq:02}
\end{equation}

Thanks to the equation above we define two aggregation-based ES basic fusion operators: $\nabla^\Sigma$  ({\em sum}) and $\nabla^{\max}$ ({\em max}), defined from the aggregation functions sum and max respectively.

Another class of ES combination operators are given in terms of projections or variants of this. The operators in this class are called {\em projective-based ES fusion operators}. In order to define this class, from now on, $d_N$ denotes $\max(N)$ for each finite society $N$ of agents in $\mathcal{S}$.

The first one of such operators is determined by projections over profiles of epistemic states. This operator is called {\em projective ES basic fusion operator}, $\nabla^\pi$ ({\em projective operator} by abuse of notation),  and is defined as follows:

\begin{description}
  \item[Projective ES basic fusion operator] $$\nabla^\pi(\Phi,\succeq)=\succeq^{\mathrm{lex}(\succeq,\succeq_{d_N})}$$
\end{description}

For this  operator, the output of a merging process depends completely on the beliefs of a single agent involved in such merging process and the integrity constraints.

  Two variants of the projective operator are obtained through a ``linearization'' of it.  These families of operators are called {\em linearized projective ES basic fusion operators} and {\em quasilinearized projective ES basic fusion operators} (called respectively {\em linearized projective operators} and {\em quasilinearized projective operators} by abuse). In order to define these operators, let us considerer $\geq$, a fixed linear order over $\Val$:

\begin{description}
  \item[Linearized projective ES basic fusion operator] $$\nabla^{\pi_\geq}(\Phi,\succeq)=\succeq^{\mathrm{lex}(\succeq,\succeq^{\pi_\geq}_\Phi)}$$
\end{description}
where $\succeq^{\pi_\geq}_\Phi=\succeq^{\mathrm{lex}(\succeq_{d_N},\geq)}$.\\

\begin{description}
  \item[Quasilinearized projective ES basic fusion operator] $$\nabla^{\mathrm{Q}\pi_\geq}(\Phi,\succeq)=\succeq^{\mathrm{lex}(\succeq,\succeq^{\mathrm{Q}\pi_\geq}_\Phi)}$$
\end{description}
where $\succeq^{\mathrm{Q}\pi_\geq}_\Phi=\succeq_i$ if $\Phi$ has a single input, namely $\succeq_i$, and $\succeq^{\mathrm{Q}\pi_\geq}_\Phi=\succeq^{\mathrm{lex}(\succeq_{d_N},\geq)}$ otherwise.\\

The following  projective-based ES fusion operator is defined using sum and projections. This operator will be  called {\em $\Sigma$-pseudoprojective ES fusion operator}, $\nabla^{\Sigma-\mathrm{P}\pi}$:

\begin{description}
  \item[$\Sigma$-pseudoprojective ES fusion operator] $$\nabla^{\Sigma-\mathrm{P}\pi}(\Phi,\succeq)=\succeq^{\mathrm{lex}(\succeq,\succeq^{\Sigma-\mathrm{P}\pi}_\Phi)}$$
\end{description}
where $\succeq^{\Sigma-\mathrm{P}\pi}_\Phi=\succeq^{\mathrm{lex}(\succeq_{d_N},\succeq^\Sigma_\Phi)}$.\\

In this kind of merging process an agent allows a very weak participation of all agents involved in the merging process.

%
%

Like sum and max operators, the projective, the linearized projective, the quasilinearized projective and \linebreak$\Sigma$-pseudoprojective  operators are ES basic fusion operators, being $\Phi\mapsto\succeq_{d_N}$, $\Phi\mapsto\succeq^{\pi_\geq}_\Phi$, $\Phi\mapsto\succeq^{\mathrm{Q}\pi_\geq}_\Phi$ and $\Phi\mapsto\succeq_\Phi^{\Sigma-\mathrm{P}\pi}$ 
their respective basic assignments.

\subsection{Behavior of our examples}

 In this section we make an exhaustive study of the rational behavior of the previously defined operators (sum, max, projective, linearized projective, quasilinearized projective and $\Sigma$-pseudoprojective operators). In particular, we show which postulates from the set of  postulates  previously established  are satisfied.

From the properties of sum and max it is easy to see that sum operator is really an ES fusion operator, unlike max operator which is simply an ES quasifusion operator because it satisfaces \sintviiiw, but \sintviii\ does not hold.

Also let us note that operators defined by aggregation functions satisfy \sintSD. This is because they are structure preserving and the freedom one has for building different preorders through its assignment.

Note that the sum operator satisfies \sintU, since \sintvii\ and \sintviii\ hold (see Proposition \ref{prop6}). Moreover, it is easy to verify that the max operator also satisfies the unanimity condition but does not satisfy \sintviii, showing that \sintvii\ and \sintviii\ are together a sufficient but  not a necessary condition for \sintU\ to hold.

Also, sum and max operators satisfy \sintP. This is because  \sintvii\ and \sintviiiw\ hold (see Proposition \ref{prop16}).

Furthermore, under the symmetry  property of aggregation functions, it is easy to see that aggregation-based ES basic fusion operators are not dictatorial operators. In particular, sum and max operators are not dictatorial operators. Thus, by the Main Impossibility Theorem, these operators do not satisfy \sintI, because \sintSD\ and \sintP\ hold. Moreover,  since these operators satisfy \sintU\ and \sintP\ but do not satisfy \sintI, we have that, modulo  the Unanimity condition,   the Pareto condition does not entail Independence condition; that is, one of the converses of Proposition \ref{prop10} does not hold.

The projective  operator satisfies \sintv, \sintvii\ and \sintviiiw, but \sintvi\ and \sintviii\ do not hold, as we show  in Proposition \ref{prop17} (see \ref{apendice}).  Otherwise, as we can see in Proposition \ref{prop21} (see \ref{apendice}), linearized projective operators satisfy \sintvii, \sintviii\ and \sintviiiw,  but do not satisfy neither \sintv\ nor \sintvi. Moreover, as we show in Proposition \ref{prop20} (see \ref{apendice}), quasilinearized projective operators satisfy \sintv\ and \sintviiiw\ but do not satisfy neither \sintvi, \sintvii\ nor \sintviii. The $\Sigma$-pseudoprojective operator satisfies \sintv\ and \sintvi, but \sintvii, \sintviii\ and \sintviiiw\ do not hold (see \ref{apendice}, Proposition \ref{prop18}).


In Propositions \ref{prop17}, \ref{prop21}, \ref{prop20} and \ref{prop18}, we also show that the projective-based ES fusion operators previously defined are dictatorial operators, that is, satisfy \sintD. Thus, by Proposition  \ref{prop12}, the operators in this class of ES basic fusion operators also satisfy the Pareto condition, that is, \sintP\ holds. Moreover, in these results we can see that the projective, quasilinearized projective and the $\Sigma$-pseudoprojective operators satisfy \sintSD, unlike linearized projective operators that do not satisfy it. There we also show that the projective, linearized and quasilinearized projective operators satisfy the Independence condition, \sintI, unlike the $\Sigma$-pseudoprojective operator which does not satisfy it.

Furthermore, Propositions \ref{prop17}, \ref{prop21} and \ref{prop18} also show respectively that, the projective, linearized projective and the $\Sigma$-pseudoprojective operators satisfy the Unanimity condition, \sintU, unlike  quasilinearized projective operators which do not satisfy it, as we show in Proposition \ref{prop20}. Thus, satisfaction of the Pareto condition does not imply satisfaction of the Unanimity condition, even if the Independence condition holds, that is, the other converse of  Proposition \ref{prop10} does not hold.

Table \ref{tab1} shows the properties satisfied by the previously defined operators in a synthesized way, offering a landscape of the rational properties that satisfy those operators. In this table, and for the operators pointed  in the first column, \cmark shows the satisfaction of
the property indicated at the top of its column, while \xmark\  denotes that such property does not hold.

\begin{table}[t]
  \centering
  \caption{Properties of the ES basic fusion operators previously defined}\label{tab1}\smallskip
\begin{tabular}{lcccccccccc}
  \hline
\tiny Operator & \tiny \sintv & \tiny \sintvi & \tiny \sintvii & \tiny \sintviii & \tiny \sintviiiw & \tiny \sintSD & \tiny \sintU & \tiny \sintP & \tiny \sintI & \tiny \sintD\\
  \hline
$\nabla^\Sigma$ & \cmark & \cmark & \cmark & \cmark & \cmark & \cmark & \cmark & \cmark & \xmark & \xmark \\
$\nabla^{\max}$ & \cmark & \cmark & \cmark & \xmark & \cmark & \cmark & \cmark & \cmark & \xmark & \xmark \\
$\nabla^\pi$ & \cmark & \xmark & \cmark & \xmark & \cmark & \cmark & \cmark & \cmark & \cmark & \cmark \\
$\nabla^{\pi_\geq}$ & \xmark & \xmark & \cmark & \cmark & \cmark & \xmark & \cmark & \cmark & \cmark & \cmark \\
$\nabla^{\mathrm{Q}\pi_\geq}$ & \cmark & \xmark & \xmark & \xmark & \cmark & \cmark & \xmark & \cmark & \cmark & \cmark \\
$\nabla^{\Sigma-\mathrm{P}\pi}$  & \cmark & \cmark & \xmark & \xmark & \xmark & \cmark  & \cmark & \cmark & \xmark & \cmark \\
\hline
\end{tabular}
\end{table}

As the Main impossibility theorem establishes, in order to see that an ES fusion basic operator is dictatorial, it is enough to see that it satisfies \sintSD, \sintP\ and \sintI. However, as shown in Table \ref{tab1}, the $\Sigma$-pseudoprojective operator is an example of a dictatorial operator that does not satisfy the Independence condition, \sintI. Furthermore, we can see that linearized projective operators are dictatorial operators which do not satisfy \sintSD. All this shows in a finer way that these three properties are not necessary conditions for an ES basic fusion operator to be dictatorial.

 }
\section{Final remarks and perspectives}\label{conclu}


We have presented an epistemic version of postulates of merging which allow us to give very precise representation theorems
(Theorems \ref{teo1}, \ref{teo2}, \ref{teo3} and Proposition \ref{prop3}). We have also introduced new postulates with a more social flavor.
These postulates, except Standard Domain, have been characterized semantically (Propositions \ref{prop7}, \ref{prop8}, \ref{prop9}, \ref{prop11}).
We have shown some tight relationships between the merging postulates and the social postulates; in particular, the ES quasifusion operators
satisfy the Pareto Condition  (Proposition \ref{prop16}) and the the ES fusion operators
satisfy the Unanimity Condition (Proposition \ref{prop6}).
Actually, unlike the other social postulates the Standard Domain Postulate has a formulation which is very near to a semantical one as the Observation \ref{remark-shapes} establishes.

The social postulates together with the basic fusion postulates are enough to prove a general impossibility theorem: the ES basic fusion operator satisfying the Standard Domain, Pareto and Independence Conditions are indeed dictatorial operators (Theorem \ref{teo6}). Moreover, the ES fusion operators
satisfying  the Standard Domain Postulate and the Independence Condition are dictatorial operators (Corollary \ref{cor1}).

One very interesting feature of our approach is that it gives interesting instantiations of the Inpossibility Theorem  (Theorem \ref{teo6}) for different representations of epistemic states. Thus,
this applies to Ordinal Conditional Functions, rational relations, and of course total preorders. However, the representation of epistemic states as formulas does not work because it is impossible to have Standard Domain in presence of a good representation of beliefs, namely the Maximality Condition. This fact highlights the necessity of using complex epistemic states.

The concrete examples of Section \ref{examples} show that the Standard Domain and Independence  are not necessary conditions in order to have a dictatorial operator (see Table \ref{tab1}).

Having established the impossibility Theorem, it is natural to ask for general results of manipulability. We have some work in progress in this direction.

At the moment, we have not  an example of an ES fusion or quasifusion operator which is at the same time a dictatorial operator. We conjecture that such kind of operators exist.

We think that Theorem \ref{new-theorem-formulas} can be generalized to an arbitrary number of models greater than 4. However the combinatorial analysis
seems very complicated.

Finally, an interesting question is to find a characterization of dictatorial operators.

\section*{Acknowledgements}
Thanks to the CDCHTA-ULA for its financial support to the Project N$^\circ$ H-1538-16-05-C. This work is a result of this project.

Thanks to Professor Olga Porras for the English proofreading.

\bibliographystyle{plain}
\bibliography{biblio}

\begin{thebibliography}{10}

\bibitem{AGM85}
Carlos~E. Alchourr{\'o}n, Peter G{\"a}rdenfors, and David Makinson.
\newblock On the logic of theory change: Partial meet contraction and revision
  functions.
\newblock {\em J. Symb. Log.}, 50(2):510--530, 1985.

\bibitem{Ar63}
Kenneth~J. Arrow.
\newblock {\em Social choice and individual values}.
\newblock Yale University Press, 1963.

\bibitem{BKPP00}
Salem Benferhat, Sébastien Konieczny, Odile Papini, and Ram{\'o}n
  Pino~P{\'e}rez.
\newblock Iterated revision by epistemic states: Axioms, semantics and syntax.
\newblock In Werner Horn, editor, {\em ECAI}, pages 13--17. IOS Press, 2000.

\bibitem{CK02}
Donald~E. Campbell and Jerry~S. Kelly.
\newblock Impossibility theorems in the arrovian framekork.
\newblock In Kotaro~Suzumura Kenneth J.~Arrow, Amartya K.~Sen, editor, {\em
  Handbook of Social Choice and Welfare, Volume 1 (Handbooks in Economics)},
  pages 35--94. North-Holland, 2002.

\bibitem{CGM06}
Samir Chopra, Aditya~K. Ghose, and Thomas~Andreas Meyer.
\newblock Social choice theory, belief merging, and strategy-proofness.
\newblock {\em Information Fusion}, 7(1):61--79, 2006.

\bibitem{DP94}
Adnan Darwiche and Judea Pearl.
\newblock On the logic of iterated belief revision.
\newblock In Ronald Fagin, editor, {\em TARK}, pages 5--23. Morgan Kaufmann,
  1994.

\bibitem{DP97}
Adnan Darwiche and Judea Pearl.
\newblock On the logic of iterated belief revision.
\newblock {\em Artificial intelligence}, 89:1--29, 1997.

\bibitem{EKM07}
Patricia Everaere, Sébastien Konieczny, and Pierre Marquis.
\newblock The strategy-proofness landscape of merging.
\newblock {\em Journal of Artificial Intelligence Research (JAIR)}, 28:49--105,
  2007.

\bibitem{G88}
Peter G\"ardenfors.
\newblock {\em Knowledge in Flux: Modeling the Dynamics of Epistemic States}.
\newblock MIT Press, Cambridge, MA, USA, 1988.

\bibitem{Gib73}
Allan Gibbard.
\newblock Manipulation of voting schemes: A general result.
\newblock {\em Econometrica}, 41(4):587--601, July 1973.

\bibitem{KM92}
Hirofumi Katsuno and Alberto~O. Mendelzon.
\newblock Propositional knowledge base revision and minimal change.
\newblock {\em Artif. Intell.}, 52(3):263--294, 1992.

\bibitem{KLM04}
Sébastien Konieczny, Jérôme Lang, and Pierre Marquis.
\newblock {DA$^{2}$} merging operators.
\newblock {\em Artificial Intelligence}, 157:49--79, 2004.

\bibitem{KP02}
Sébastien Konieczny and Ramón Pino~Pérez.
\newblock Merging information under constraints: A logical framework.
\newblock {\em J. Log. Comput.}, 12(5):773--808, 2002.

\bibitem{KP05}
Sébastien Konieczny and Ramón Pino~Pérez.
\newblock Propositional belief base merging or how to merge beliefs/goals
  coming from several sources and some links with social choice theory.
\newblock {\em European Journal of Operational Research}, 160(3):785--802,
  2005.

\bibitem{KP11}
Sébastien Konieczny and Ramón Pino~Pérez.
\newblock Logic based merging.
\newblock {\em J. Philos. Logic}, 40(2):239--270, 2011.

\bibitem{KP98}
S{\'e}bastien Konieczny and Ram{\'o}n Pino~P{\'e}rez.
\newblock On the logic of merging.
\newblock In Anthony~G. Cohn, Lenhard~K. Schubert, and Stuart~C. Shapiro,
  editors, {\em KR}, pages 488--498. Morgan Kaufmann, 1998.

\bibitem{KP99}
S{\'e}bastien Konieczny and Ram{\'o}n Pino~P{\'e}rez.
\newblock Merging with integrity constraints.
\newblock In Anthony Hunter and Simon Parsons, editors, {\em ESCQARU}, volume
  1638 of {\em Lecture Notes in Computer Science}, pages 233--244. Springer,
  1999.

\bibitem{LM92}
D.~Lehmann and M.~Magidor.
\newblock What does a conditional knowledge base entail?
\newblock {\em Artificial Intelligence}, 55:1--60, 1992.

\bibitem{MP11}
Amílcar Mata~Díaz and Ramón Pino~Pérez.
\newblock Logic-based fusion of complex epistemic states.
\newblock In Weiru Liu, editor, {\em ECSQARU}, volume 6717 of {\em Lecture
  Notes in Computer Science}, pages 398--409. Springer, 2011.

\bibitem{M00}
Thomas~Andreas Meyer.
\newblock Merging epistemic states.
\newblock In Riichiro Mizoguchi and John~K. Slaney, editors, {\em PRICAI},
  volume 1886 of {\em Lecture Notes in Computer Science}, pages 286--296.
  Springer, 2000.

\bibitem{Sat75}
Mark~Allen Satterthwaite.
\newblock Strategy-proofness and arrow's conditions: Existence and
  correspondence theorems for voting procedures and social welfare functions.
\newblock {\em Journal of Economic Theory}, 10(2):187--217, April 1975.

\bibitem{Spo88}
Wolfgang Spohn.
\newblock Ordinal conditional functions: A dynamic theory of epistemic states.
\newblock In W.~L. Harper and B.~Skyrms, editors, {\em Causation in Decision,
  Belief Change and Statistics}, pages 105--134. Kluwer, August 1988.

\bibitem{Suz02}
Kotaro Suzumura.
\newblock Introduction.
\newblock In Kotaro~Suzumura Kenneth J.~Arrow, Amartya K.~Sen, editor, {\em
  Handbook of Social Choice and Welfare, Volume 1 (Handbooks in Economics)},
  pages 1--32. North-Holland, 2002.

\end{thebibliography}

\appendix

\section{Proofs}\label{apendice}

{\bf Proof of Proposition \ref{prop3}:}
  Let $\Phi\mapsto\succeq_\Phi$ be an assignment that satisfies  Properties \assgiii\ and \assgiv. If this assignment satisfies \assgii\,  the maximality condition follows straightforwardly. In order to prove the converse, we suppose that $\Phi\mapsto\succeq_\Phi$ satisfies the maximality condition, and using induction on the length of finite societies, we will show that for any $N$ in $\mathcal{F}^\ast(\mathcal{S})$, $N$ satisfies the following:
\begin{equation}\label{eq:03}
  \mbox{For any }N\mbox{-profile }\Phi, \mbox{ if }\textstyle{\Bigwedge_{i\in N}B(E_i)\not\vdash\bot}, \mbox{ then }\textstyle{\model{\Bigwedge_{i\in N} B(E_{i})}=\max(\succeq_\Phi)}
\end{equation}
If $N$ has a single agent, the above property is exactly the maximality condition. Now, let $n$ be a  positive integer, with $n>1$, and suppose that for all $D$ in $\mathcal{F}^\ast(\mathcal{S})$, with less than $n$ agents, the condition \eqref{eq:03} holds. Assume $N$ in $\mathcal{F}^\ast(\mathcal{S})$ consisting of $n$ agents. Let us consider $\Phi$ an $N$-profile satisfying $\Bigwedge_{i\in N} B(E_{i})\not\vdash\bot$. Let $w$ and $w'$ be interpretations. Assume that $w\models \Bigwedge_{i\in N} B(E_{i})$. Let us take $D=N\setminus\{j\}$ a finite society of agents in $\mathcal{S}$, where $j$ is an agent in $N$. Thus, the $D$-profile $\Phi\upharpoonright_D$ satisfies  $w\models \Bigwedge_{i\in D} B(E_{i})$. Then, by condition \eqref{eq:03}, $w\succeq_{\Phi\upharpoonright_D} w'$. Since $w\models B(E_j)$, by the maximality condition, we have $w\succeq_{E_j} w'$. Hence, from  Property \assgiii, it follows that $w\succeq_\Phi w'$, that is, $w\in\max(\succeq_\Phi)$. Thus, $\model{\Bigwedge_{i\in N} B(E_{i})}\subseteq \max(\succeq_\Phi)$.

Now we prove that $\max(\succeq_\Phi)\subseteq \model{\Bigwedge_{i\in N} B(E_{i})}$. Let $w$ be in $\max(\succeq_\Phi)$ and, towards a contradiction, suppose $w\not\models\Bigwedge_{i\in N} B(E_{i})$. Thus, there exists $j$ in $N$ such that $w\not \models B(E_j)$. Let us consider the finite society $D=N\setminus\{j\}$. Take an interpretation $w'$ such that $w'\models\Bigwedge_{i\in N}B(E_{i})$. Clearly, we have
$w'\models \Bigwedge_{i\in D} B(E_{i})$. From condition \eqref{eq:03}, we obtain
$w'\succeq_{\Phi\upharpoonright_D} w$. Since $w'\models B(E_j)$ and $w\not\models B(E_j)$, from the maximality condition we have $w'\succ_{E_j} w$. Therefore,
 by Property \assgiv, we have $w'\succ_\Phi w$, a contradiction.\qed

{\bf Proof of Theorem \ref{teo1}:}

\noindent({\em Only if part})\quad Assume that $\nabla$ is an ES basic fusion operator. We define the assignment $\Phi\mapsto\succeq_\Phi$ associated to this operator as follow: for all $N$ in $\mathcal{F}^\ast(\mathcal{S})$ and every $N$-profile $\Phi$, $\succeq_\Phi$ is given by:
    \begin{equation}\label{eq:8}
    w\succeq_\Phi w'\mbox{\it iff } w\models B\big(\nabla(\Phi,E_{w,w'})\big);
    \end{equation}
    where  $E_{w,w'}$ is an epistemic state in  $\mathcal{E}$  satisfying $\model{B(E_{w,w'})}=\{w,w'\}$.

Since $B(\mathcal{E})$ is exactly $\mathcal{L_P}^\ast\!/\!\!\equiv$, modulo equivalence, there exists such an epistemic state $E_{w,w'}$ in $\mathcal{E}$. Moreover, because of \sintii, the definition of $\succeq_\Phi$ does not depend on the choice of $E_{w,w'}$. Thus, let us show that $\succeq_\Phi$ is a total preorder over $\mathcal{W_P}$.

\begin{description}
        \item[Total] Let $w$, $w'$ be interpretations in $\mathcal{W_P}$ and consider $E_{w,w'}$ an epistemic state that satisfies $\model{B(E_{w,w'})}=\{w,w'\}$. From \sinti\ and the consistency of $B\big(\nabla(\Phi,E_{w,w'})\big)$, we have either $w\models B\big(\nabla(\Phi,E_{w,w'})\big)$ or $w'\models B\big(\nabla(\Phi,E_{w,w'})\big)$. By definition, this means that $w\succeq_\Phi w'$ or $w'\succeq_\Phi w$, respectively.

        \item[Transitivity] Assume that $w$, $w'$ and $w''$ are interpretations in $\mathcal{W_P}$ and suppose that $w\succeq_\Phi w'$ and $w'\succeq_\Phi w''$. We want to show that $w\succeq_\Phi w''$. Towards a contradiction, assume that $w\not\succeq_\Phi w''$. Let $E_{w,w'}$, $E_{w',w''}$, $E_{w,w''}$, $E$ be epistemic states  satisfying $\model{B(E_{w,w'})}=\{w,w'\}$, $\model{B(E_{w',w''})}=\{w',w''\}$, $\model{B(E_{w,w''})}=\{w,w''\}$ and $\model{B(E)}=\{w,w',w''\}$. Since $B\big(\nabla(\Phi,E_{w,w''})\big)$ is consistent and $w\not\succeq_\Phi w''$, from \sinti\ we have that $w''$ is the unique model of  $B\big(\nabla(\Phi,E_{w,w''})\big)$. Now we consider the following two cases:\\

            \begin{itemize}
            \item $B\big(\nabla(\Phi,E)\big)\wedge B(E_{w,w''})\not\vdash\bot$. In this case, since  $B(E_{w,w''})\equiv B(E)\wedge B(E_{w,w''})$,
                     by \sintiii\ and \sintiv, we have
                \begin{equation}\label{eq:nueva}
                B\big(\nabla(\Phi,E)\big)\wedge B(E_{w,w''})\equiv B\big(\nabla(\Phi,E_{w,w''})\big)
                \end{equation}
                From the fact that $w\not\models B\big(\nabla(\Phi,E_{w,w''})\big)$ and the equivalence \eqref{eq:nueva},   we have  $w\not\models B\big(\nabla(\Phi,E)\big)$. From the equivalence \eqref{eq:nueva}, we have also $w''\models B\big(\nabla(\Phi,E)\big)$. Then, by \sintiii\ and \sintiv\ we have
            \begin{equation}\label{eq:09}
            B\big(\nabla(\Phi,E)\big)\wedge B(E_{w',w''})\equiv B\big(\nabla(\Phi,E_{w',w''})\big)
            \end{equation}

            Moreover, $w'\succeq_\Phi w''$ means that $w'\models B\big(\nabla(\Phi,E_{w',w''})\big)$. So, from the equivalence \eqref{eq:09}, we have that $w'\models B\big(\nabla(\Phi,E)\big)$. Therefore, $\model{B\big(\nabla(\Phi,E)\big)}=\{w',w''\}$. Thus, by \sintiii\ and \sintiv, we have
            \begin{equation}\label{eq:10}
            B\big(\nabla(\Phi,E)\big)\wedge B(E_{w,w'})\equiv B\big(\nabla(\Phi,E_{w,w'})\big)
            \end{equation}
            But $w\succeq_\Phi w'$, then, by definition and the equivalence \eqref{eq:10}, we have  $w\models B\big(\nabla(\Phi,E)\big)$, a contradiction.\\

            \item $B\big(\nabla(\Phi,E)\big)\wedge B(E_{w,w''})\vdash\bot$. In this case, by \sinti, $w'$ is the sole model of $B\big(\nabla(\Phi,E)\big)$. Therefore, $B\big(\nabla(\Phi,E)\big)\wedge B(E_{w,w'})\not\vdash\bot$ and by \sintiii\ and \sintiv, we have
            \begin{equation}\label{eq:11}
            B\big(\nabla(\Phi,E)\big)\wedge B(E_{w,w'})\equiv B\big(\nabla(\Phi,E_{w,w'})\big)
            \end{equation}

            Since $w\not\models B\big(\nabla(\Phi,E)\big)$, by the equivalence \eqref{eq:11}, we have
            $w\not\models B\big(\nabla(\Phi,E_{w,w'})\big)$, that is,  $w\not\succeq_\Phi w'$, a contradiction.
            \end{itemize}
        \end{description}

        Now we show that $\Phi\mapsto\succeq_\Phi$ satisfies \eqref{B-Rep}. In order to do this, we consider $N$ in $\mathcal{F}^\ast(\mathcal{S})$, an $N$-profile $\Phi$ and an epistemic state $E$ in $\mathcal{E}$ and we have to verify that $\model{B\big(\nabla(\Phi,E)\big)}=\max\big(\model{B(E)},{\succeq_\Phi}\big)$.

        \begin{itemize}
        \item First we prove that $\model{B\big(\nabla(\Phi,E)\big)}\subseteq \max\big(\model{B(E)},{\succeq_\Phi}\big)$. Let $w$ be an interpretation in $\mathcal{W_P}$ such that \linebreak$w\models B\big(\nabla(\Phi,E)\big)$. By  \sinti, $w\models B(E)$. Towards a contradiction, suppose that $w$ is not in $\max\big(\model{B(E)},\succeq_\Phi\big)$. Thus, there exists $w'$, a model of $B(E)$ such that $w'\succ_\Phi w$.  Let $E_{w,w'}$ be an epistemic state which satisfies $\model{B(E_{w,w'})}=\{w,w'\}$. By the definition given by the equivalence \eqref{eq:8}, we have  $w\not\models B\big(\nabla(\Phi,E_{w,w'})\big)$. Then, since $B\big(\nabla(\Phi,E)\big)\wedge B(E_{w,w'})\not\vdash\bot$ and $B(E_{w,w'})\equiv B(E)\wedge B(E_{w,w'})$, by \sintiii\ and  \sintiv, we have $B\big(\nabla(\Phi,E)\big)\wedge B(E_{w,w'})\equiv B\big(\nabla(\Phi,E_{w,w'})\big)$. Therefore $w\not\models B\big(\nabla(\Phi,E)\big)$, a contradiction.\\

        \item Now we prove that $\model{B\big(\nabla(\Phi,E)\big)}\supseteq \max\big(\model{B(E)]},{\succeq_\Phi}\big)$. Consider $w$, $w'$ a pair of interpretations such that $w$ is in $\max\big(\model{B(E)},\succeq_\Phi\big)$ and $w'\models B\big(\nabla(\Phi,E)\big)$. By \sinti, $w'\models B(E)$, and then $w\succeq_\Phi w'$.
            Let $E_{w,w'}$ be an epistemic state satisfying $\model{B(E_{w,w'})}=\{w,w'\}$.
            By definition,  $w\models B\big(\nabla(\Phi,E_{w,w'})\big)$.  Since $B\big(\nabla(\Phi,E)\big)$ is consistent with $B(E_{w,w'})$, by \sintiii\ and \sintiv, we have $B\big(\nabla(\Phi,E)\big)\wedge B(E_{w,w'})\equiv B\big(\nabla(\Phi,E_{w,w'})\big)$.
            From this and the fact that $w\models B\big(\nabla(\Phi,E_{w,w'})\big)$, we have  $w\models B\big(\nabla(\Phi,E)\big)$.

        \end{itemize}

Finally, from \sintii\ it follows straightforwardly  that $\Phi\mapsto\succeq_\Phi$ is really a basic assignment. Moreover, it is easy to see that this is the only basic assignment  that satisfies \eqref{B-Rep}. Indeed, any  assignment satisfying
\eqref{B-Rep} has to satisfy also the equivalence \eqref{eq:8}, so it is unique.

\bigskip
\noindent({\em If part})\quad Assume that $\Phi\mapsto\succeq_\Phi$ satisfies \eqref{B-Rep}. The satisfaction of \sinti\ and \sintii\ is a straightforward consequence of \eqref{B-Rep} and by the fact that $\Phi\mapsto\succeq_\Phi$ is a basic assignment. Thus, it remains to prove \sintiii\ and \sintiv.

\begin{description}
    \item[\sintiii] Suppose we have a finite society of agents $N$ in $\mathcal{F}^\ast(\mathcal{S})$, $\Phi$ an $N$--profile, and $E$, $E'$, $E''$ three  epistemic states such that $B(E)\equiv B(E')\wedge B(E'')$. We must show that
        $$B\big(\nabla(\Phi,E')\big)\wedge B(E'')\vdash B\big(\nabla(\Phi,E)\big)$$

        If $B\big(\nabla(\Phi,E')\big)\wedge B(E'')\vdash\bot$ the above entailment is clear.
         Now suppose that $B\big(\nabla(\Phi,E')\big)\wedge B(E'')\not\vdash\bot$. Let $w$ be an interpretation that satisfies $w\models B\big(\nabla(\Phi,E')\big)\wedge B(E'')$. Thus, from \eqref{B-Rep} we have that $w\models B(E')\wedge B(E'')$, and therefore $w\models B(E)$. Let $w'$ be any interpretation such that $w'\models B(E)$; then $w'\models B(E')$. Since $w\models B\big(\nabla(\Phi,E')\big)$, by \eqref{B-Rep} we have  $w\succeq_\Phi w'$; and again  by \eqref{B-Rep}, we have $w\models B\big(\nabla(\Phi,E)\big)$.\\

    \item[\sintiv] Let us consider $N$ in $\mathcal{F}^\ast(\mathcal{S})$, an $N$-profile $\Phi$ and epistemic states $E$, $E'$ and $E''$ such that $B(E)\equiv B(E')\wedge B(E'')$ and $B\big(\nabla(\Phi,E')\big)\wedge B(E'')\not\vdash\bot$. We want to prove that $B\big(\nabla(\Phi,E)\big)\vdash B\big(\nabla(\Phi,E')\big)\wedge B(E'')$. Let $w$ be an interpretation such that $w\models B\big(\nabla(\Phi,E)\big)$. Towards a contradiction, suppose that $w\not\models B\big(\nabla(\Phi,E')\big)\wedge B(E'')$. Let $w'$ be a model of $B\big(\nabla(\Phi,E')\big)\wedge B(E'')$. Thus,  by \eqref{B-Rep}, $w'\models B(E')\wedge B(E'')$, and then $w'\models B(E)$. Since $w\models B(E')$, $w\not\models B\big(\nabla(\Phi,E')\big)$ and $w'\models B\big(\nabla(\Phi,E')\big)$, by \eqref{B-Rep}, we have   $w'\succ_\Phi w$. Again by \eqref{B-Rep}, this means that $w\not\models B\big(\nabla(\Phi,E)\big)$, a contradiction.\qed
    \end{description}

{\bf Proof of Proposition \ref{prop4}:}
Let $\nabla$ be an ES basic fusion operator and let us consider $\Phi\mapsto\succeq_\Phi$ the basic assignment given by theorem \ref{teo1} associated to $\nabla$.

\begin{enumerate}
\item[(i)] ({\em Only if part}) Let us suppose that $\nabla$ satisfies \sintv. We want to show that $\Phi\mapsto\succeq_\Phi$ satisfies \assgi. Let $j$ and $k$ be a pair of agents in $\mathcal{S}$, and consider profiles $E_j$, $E_k$, such that $E_j\neq E_k$. We want to see that  $\succeq_{E_j}\neq\succeq_{E_k}$.
     By \sintv, there exists an epistemic state $E$ such that $B\big(\nabla(E_j,E)\big)\not\equiv B\big(\nabla(E_k,E)\big)$. From this, as a straightforward consequence of \eqref{B-Rep}, we have that $\succeq_{E_j}\neq\succeq_{E_k}$.\\

    ({\em If part}) If $\Phi\mapsto\succeq_\Phi$ satisfies \assgi, there exists a pair of interpretations $w$, $w'$ such that $w\succeq_{E_j}w'$ and $w'\succ_{E_k}w$. Let $E$ be an epistemic state such that $\model{B(E)}=\{w,w'\}$, then  $\max\big(\model{B(E)},{\succeq_{E_j}}\big)\neq \max\big(\model{B(E)},{\succeq_{E_k}}\big)$. Thus, by \eqref{B-Rep} we have that $B\big(\nabla(E_j,E)\big)\not\equiv B\big(\nabla(E_k,E)\big)$.\\

\item[(ii)] ({\em Only if part}) Assume that $\nabla$ satisfies \sintvi. Let us consider $N$ in $\mathcal{F}^\ast(\mathcal{S})$ and   $\Phi$ an $N$-profile such that \linebreak$\Bigwedge_{i\in N} B(E_i)\not\vdash\bot$.   We want   to show   that $\max(\succeq_\Phi)=\model{\Bigwedge_{i\in N} B(E_i)}$.

    \begin{itemize}
      \item First we prove $\max(\succeq_\Phi)\subseteq\model{\Bigwedge_{i\in N} B(E_i)}$. In order to do this, let $w$ be an interpretation in $\max(\succeq_\Phi)$, and suppose, towards a contradiction, that $w\not\models\Bigwedge_{i\in N} B(E_i)$. Let $w'$ be a model of $\Bigwedge_{i\in N} B(E_i)$. Note that $w\succeq_\Phi w'$.
          Let $E$ be an epistemic state such that $\model{B(E)}=\{w,w'\}$. Then, because of \sintvi, we have  that $B\big(\nabla(\Phi,E)\big)\equiv\Bigwedge_{i\in N} B(E_i)\wedge B(E)$. Thus, $w\not\models B\big(\nabla(\Phi,E)\big)$ and $w'\models B\big(\nabla(\Phi,E)\big)$. Then, by \eqref{B-Rep}, we have $w'\succ_\Phi w$, a contradiction.

      \item Now we prove $\max(\succeq_\Phi)\supseteq\model{\Bigwedge_{i\in N} B(E_i)}$. Let $w$ be a model of  $\Bigwedge_{i\in N} B(E_i)$. We want to see that $w$ is in $\max(\succeq_\Phi)$. Let  $w'$ be an interpretation and $E$ be an epistemic state such that $\model{B(E)}=\{w,w'\}$. Since $w$ is a model of $\Bigwedge_{i\in N} B(E_i)\wedge B(E)$,  by \sintvi, we have  $w\models B\big(\nabla(\Phi,E)\big)$. From this and  \eqref{B-Rep}, we have  $w\succeq_\Phi w'$.

    \end{itemize}
({\em If part}) Assume that $\Phi\mapsto\succeq_\Phi$ satisfies Property \assgii. In order to show that $\nabla$ satisfies \sintvi,  consider $N$ in $\mathcal{F}^\ast(\mathcal{S})$,  an $N$--profile $\Phi$ and an epistemic state $E$ such that $\Bigwedge_{i\in N} B(E_i)\wedge B(E)\not\vdash\bot$.
In particular, $\Bigwedge_{i\in N} B(E_i)\not\vdash\bot$.
Thus, by Property \assgii, $\max(\succeq_\Phi)=\model{\Bigwedge_{i\in N} B(E_i)}$. Therefore, $\max\big(\model{B(E)},\succeq_\Phi\big)=\model{\Bigwedge_{i\in N} B(E_i)\wedge B(E)}$. Then, by \eqref{B-Rep},  $B\big(\nabla(\Phi,E)\big)\equiv\Bigwedge_{i\in N} B(E_i)\wedge B(E)$.\\

\item[(iii)] ({\em Only if part}) We assume \sintvii. In order to prove Property \assgiii, consider $N$ in $\mathcal{F}^\ast(\mathcal{S})$, $\Set{N_1, N_2}$ a partition of $N$, $\Phi$ an $N$-profile and $w$, $w'$ interpretations such that $w\succeq_{\Phi\upharpoonright_{N_1}}w'$ and $w\succeq_{\Phi\upharpoonright_{N_2}}w'$. We want to show that  $w\succeq_\Phi w'$.
     Let $E$ be an epistemic state such that $\model{B(E)}=\{w,w'\}$. By \eqref{B-Rep}, we have that $w\models B\big(\nabla(\Phi\upharpoonright_{N_1},E)\big)\wedge B\big(\nabla(\Phi\upharpoonright_{N_2},E)\big)$. Then, by \sintvii,
      we have $w\models B\big(\nabla(\Phi,E)\big)$. Thus, by \eqref{B-Rep},    $w\succeq_\Phi w'$.\\

    ({\em If part}) Suppose that the assignment $\Phi\mapsto\succeq_\Phi$ satisfies Property \assgiii. In order to prove   that $\nabla$ satisfies \sintvii, consider $N$ in $\mathcal{F}^\ast(\mathcal{S})$, $\Set{N_1, N_2}$ a partition of $N$, $\Phi$ an $N$-profile and $E$ an epistemic state. If $B\big(\nabla(\Phi\upharpoonright_{N_1},E)\big)$ is inconsistent  with $B\big(\nabla(\Phi\upharpoonright_{N_2},E)\big)$ the result follows straightforwardly. Now suppose that  $B\big(\nabla(\Phi\upharpoonright_{N_1},E)\big)\wedge B\big(\nabla(\Phi\upharpoonright_{N_2},E)\big)$ is consistent, and let $w$ be a model of it.  Let $w'$ be any model of $B(E)$. By \eqref{B-Rep}, we have
    $w\succeq_{\Phi\upharpoonright_{N_1}} w'$ and $w\succeq_{\Phi\upharpoonright_{N_2}}w'$. Thus, from Property  \assgiii, it follows that $w\succeq_\Phi w'$; and this is true  for all $w'$ such that $w'\models B(E)$, that is, $w\in \max(\model{B(E)},\succeq_\Phi)$. Then, by  \eqref{B-Rep}, we have that $w\models B\big(\nabla(\Phi,E)\big)$.\\

\item[(iv)] ({\em Only if part}) Assume that $\nabla$ satisfies \sintviii. In order to see that the assignment  $\Phi\mapsto\succeq_\Phi$ satisfies Property \assgiv,  consider $N$ in $\mathcal{F}^\ast(\mathcal{S})$, $\Set{N_1, N_2}$ a partition of $N$, $\Phi$ an $N$-profile, $w$, $w'$ a pair of interpretations and $E$ an epistemic state such that $\model{B(E)}=\{w,w'\}$, $w\succeq_{\Phi\upharpoonright_{N_1}}w'$ and $w\succ_{\Phi\upharpoonright_{N_2}}w'$. Thus, by \eqref{B-Rep}, $w$ is the sole model of $B\big(\nabla(\Phi\upharpoonright_{N_1},E)\big)\wedge B\big(\nabla(\Phi\upharpoonright_{N_2},E)\big)$. By \sintviii,  $w$ is also the sole model of $B\big(\nabla(\Phi,E)\big)$. From this, again by \eqref{B-Rep}, we have $w{\succ_\Phi} w'$.\\

    ({\em If part}) Assume that the assignment $\Phi\mapsto\succeq_\Phi$ satisfies Property \assgiv. Suppose, towards a contradiction, that $\nabla$ does not satisfy \sintviii. Thus, there exist $N$ in $\mathcal{F}^\ast(\mathcal{S})$, $\Set{N_1, N_2}$ a partition of $N$, $\Phi$ an $N$-profile and $w$ a model of $B\big(\nabla(\Phi,E)\big)$ such that   $w\not\models B\big(\nabla(\Phi\upharpoonright_{N_1},E)\big)\wedge B\big(\nabla(\Phi\upharpoonright_{N_2},E)\big)$ and $B\big(\nabla(\Phi\upharpoonright_{N_1},E)\big)\wedge B\big(\nabla(\Phi\upharpoonright_{N_2},E)\big)$ is consistent. Note that $w\models B(E)$. Without loss of generality, suppose that $w\not\models B\big(\nabla(\Phi\upharpoonright_{N_2},E)\big)$, and assume $w'$ is a model of $B\big(\nabla(\Phi\upharpoonright_{N_1},E)\big)\wedge B\big(\nabla(\Phi\upharpoonright_{N_2},E)\big)$. From this, by \sintviii, we have  $w'\models B\big(\nabla(\Phi,E)\big)$ and then, by \eqref{B-Rep}, $w\succeq_\Phi w'$.
 Since $w'\models B\big(\nabla(\Phi\upharpoonright_{N_1},E)\big)\wedge B\big(\nabla(\Phi\upharpoonright_{N_2},E)\big)$ and $w\not\models B\big(\nabla(\Phi\upharpoonright_{N_2},E)\big)$, by \eqref{B-Rep}, we have  $w'{\succeq_{\Phi\upharpoonright_{N_1}}} w$ and $w'{\succ_{\Phi\upharpoonright_{N_2}}} w$. Then,
 from  Property \assgiv, it follows that $w'\succ_\Phi w$, a contradiction.\\

\item[(v)] ({\em Only if part}) Assume that $\nabla$ satisfies \sintviiiw. In order to see that the assignment $\Phi\mapsto\succeq_\Phi$ satisfies Property  \assgivp,  consider $N$ in $\mathcal{F}^\ast(\mathcal{S})$, $\Set{N_1, N_2}$ a partition of $N$, $\Phi$ an $N$-profile, and $w$, $w'$ interpretations such that, $w{\succ_{\Phi\upharpoonright_{N_1}}}w'$ and $w{\succ_{\Phi\upharpoonright_{N_2}}}w'$.
Let $E$ be an epistemic state  such that $\model{B(E)}=\{w,w'\}$. By \eqref{B-Rep}, we have $B\big(\nabla(\Phi\upharpoonright_{N_1},E)\big)\vee B\big(\nabla(\Phi\upharpoonright_{N_2},E)\big)$ has a unique model, namely $w$. Thus, by \sintviiiw, $w$ is also the unique model of $B\big(\nabla(\Phi,E)\big)$. Therefore, by \eqref{B-Rep}, we have  $w\succ_\Phi w'$.\\

({\em If part}) Assume Property \assgivp. In order to see that \sintviiiw\ holds,   consider $N$ in $\mathcal{F}^\ast(\mathcal{S})$, $\Set{N_1, N_2}$ a partition of $N$, $\Phi$ an $N$--profile and $E$ an epistemic state such that $B\big(\nabla(\Phi\upharpoonright_{N_1},E)\big)\wedge B\big(\nabla(\Phi\upharpoonright_{N_2},E)\big)\not\vdash\bot$. Let $w$ be a model of $B\big(\nabla(\Phi,E)\big)$. We want to show that $w$ is a model of $B\big(\nabla(\Phi\upharpoonright_{N_1},E)\big)\vee B\big(\nabla(\Phi\upharpoonright_{N_2},E)\big)$.
Suppose that $w\not\models B\big(\nabla(\Phi\upharpoonright_{N_1},E)\big)$.  Let $w'$ be a model of $B\big(\nabla(\Phi\upharpoonright_{N_1},E)\big)\wedge B\big(\nabla(\Phi\upharpoonright_{N_2},E)\big)$. Then, by \eqref{B-Rep}, we have that $w{\succeq_\Phi}w'$, $w'{\succ_{\Phi\upharpoonright_{N_1}}}w$, $w'{\succeq_{\Phi\upharpoonright_{N_2}}}w$. If $w'{\succ_{\Phi\upharpoonright_{N_2}}}w$, since $w'{\succ_{\Phi\upharpoonright_{N_1}}}w$, from Property \assgivp,  we have that $w'{\succ_\Phi}w$, a contradiction.
Thus, necessarily $w{\simeq_{\Phi\upharpoonright_{N_2}}}w'$, and,  by \eqref{B-Rep}, we have $w\models B\big(\nabla(\Phi\upharpoonright_{N_2},E)\big)$. Therefore $w\models
B\big(\nabla(\Phi\upharpoonright_{N_1},E)\big)\vee B\big(\nabla(\Phi\upharpoonright_{N_2},E)\big)$, as desired.\qed
\end{enumerate}

{\bf Proof of Theorem \ref{new-theorem-formulas}:}
First of all, note that there are four interpretations and therefore fifteen elements in $\Form^\ast/\!\equiv$.
Since  $\nabla$ is a basic operator,  the assignment  restraint to profiles of size one has at most fifteen
elements in its image. That is, the function $\varphi\mapsto \succeq_\varphi$ has at most fifteen images.
We suppose that the Maximality Condition holds, that is, $max(\succeq_\varphi)=\model{\varphi}$.
Table \ref{tab2} will be useful for our combinatorial analysis.

\begin{table}
\centering
\caption{Types of preorders associated to formulas according to the number of their models}\label{tab2}
\smallskip
\begin{tabular}{ccc}
\hline
\# of models & \# of formulas $\varphi$ & possible types of $\succeq_\varphi$\\ \hline \hline
4 &  1 & $\bullet\; \bullet\; \bullet\; \bullet$\\\hline
3 & 4 & \ordre{\bullet\; \bullet\; \bullet\\ \bullet}\\\hline
2 & 6 & \ordre{\bullet\; \bullet\\ \bullet\; \bullet}\ \ \ \ \ordre{\bullet\; \bullet\\ \bullet\\ \bullet}\\ \hline
1 & 4 & \ordre{\bullet\\ \bullet\\ \bullet\\ \bullet}\ \ \  \ \ordre{\bullet\\ \bullet\\ \bullet\;  \bullet} \ \ \
\ordre{\bullet\\ \bullet\; \bullet\\ \bullet}\ \ \  \ \ordre{\bullet\\ \bullet\; \bullet\; \bullet}\\ \hline
\end{tabular}
\end{table}

By Observation \ref{remark-shapes}, all the preorders $\succeq_\varphi$ have to cover all the shapes of the following types:\\
\begin{center}
\ordred{S_1}{\bullet\; \bullet\; \bullet}\hspace{8mm}
\ordred{S_2}{\bullet\; \bullet\\ \bullet}\hspace{8mm}
\ordred{S_3}{\bullet\\ \bullet\; \bullet}\hspace{8mm}
\ordred{S_4}{\bullet\\ \bullet\\ \bullet}
\end{center}
where the points represent 3 arbitrary models taken among the four models we are considering.
It is clear that the  first two shapes, \ie\ the shapes of the form $S_1$ and $S_2$, can be covered by the images $\succeq_\varphi$ of the formulas of 4 and 3 models
(actually, the image of formulas of 3 models are enough to cover all the shapes of the type $S_1$ and $S_2$).
It is also clear that the images of formulas of 4 and 3 models can not cover any shape of types $S_3$ and $S_4$.
Thus, the problem is now reduced to knowing if it is possible to cover all the shapes of types $S_3$ and $S_4$ with the images of
formulas having  2 models or  1 model. We claim that this is impossible.

First we consider all the types of total preorders which can be the image by the assignment of formulas having  2 models  ($D_1$ and $D_2$) or  1 model ($T_1,\dots , T_4$) in the following way:\\
\begin{center}
\ordred{D_1}{\bullet\; \bullet\\ \bullet\; \bullet}\hspace{8mm}
\ordred{D_2}{\bullet\; \bullet\\ \bullet\\ \bullet}\hspace{8mm}
\ordred{T_1}{\bullet\\ \bullet\\ \bullet\\ \bullet}\hspace{8mm}
\ordred{T_2}{\bullet\\ \bullet\\ \bullet\;  \bullet}\hspace{8mm}
\ordred{T_3}{\bullet\\ \bullet\; \bullet\\ \bullet}\hspace{8mm}
\ordred{T_4}{\bullet\\ \bullet\; \bullet\; \bullet}
\end{center}

In order to see the previously mentioned impossibility, note that there are 24 possible patterns for the shapes of type $S_4$   (the number of ways to select 3 elements of 4, that is $\binom43$, multiplied by the number of all possible orders with the 3 selected elements, that is $3!$). A similar analysis shows that the number of possible patterns for the shapes of type $S_3$ is 12 (the number of ways to select 2 elements of 4, that is $\binom42$, multiplied by the number of ways to select 1 among the 2 remainder elements, that is 2).

It is easy to see that the following claims hold:
\begin{itemize}
\item A total pre-order of the type $T_1$ covers 4 patterns of type $S_4$ and 0 pattern of type $S_3$.
\item A total pre-order of the type $T_2$ covers 2 patterns of type $S_4$ and 2 patterns of type $S_3$.
\item A total pre-order of the type $T_3$ covers 2 patterns of type $S_4$ and 1 pattern of type $S_3$.
\item A total pre-order of the type $T_4$ covers 0 patterns of type $S_4$ and 3 patterns of type $S_3$.
\item A total pre-order of the type $D_2$ covers 2 patterns of type $S_4$ and 0 pattern of type $S_3$.
\item  A total pre-order of the type $D_1$ covers 0 patterns of type $S_4$ and 2 patterns of type $S_3$.
\end{itemize}

Now we analyse where could the four formulas having one model be mapped under the assignment representing $\nabla$.
This information is represented by a vector $(n_1,n_2,n_3,n_4)$ of integers greater than or equal to zero, where $n_i$ is the number of formulas
with one model having image a pre-order of type $T_i$. Note that $\sum_{i=1}^4n_i=4$.
We use the same kind of vectorial representation to see where could the six formulas having two models be mapped, that is, a vector
$(m_1,m_2)$ of integers greater than or equal to zero, where $m_i$ is the number the formulas
with two models having image a pre-order of type $D_i$. Note that $m_1+m_2=6$.

In Table \ref{tab3} we show the distributions that can cover the 24 patterns for the shapes of type  $S_4$  and  in the second column appears the maximal number of
  $S_3$ patterns covered by each distribution:

\begin{table}
\centering
\caption{Table of distributions covering the  24 shapes of type $S_4$ and maximal number of shapes $S_3$ covered by them}\label{tab3}
\smallskip
\begin{tabular}{cc}
\hline
Distributions & Maximal number  of patterns \\
 & of type $S_3$ covered \\
\hline\hline
 $\Set{(4,0,0,0), (0,6)}$ & 0\\
 $\Set{(4,0,0,0), (1,5)}$ & 2\\
 $\Set{(4,0,0,0), (2,4)}$ & 4\\
 $\Set{(3,1,0,0), (0,6)}$ & 2\\
 $\Set{(3,1,0,0), (1,5)}$ & 4\\
 $\Set{(3,0,1,0), (0,6)}$ & 1\\
 $\Set{(3,0,1,0), (1,5)}$ & 3\\
 $\Set{(3,0,0,1), (0,6)}$ & 3\\
 \hline
\end{tabular}
\end{table}

Therefore,  each distribution covering the 24 patterns of type $S_4 $, covers strictly less than   12 patterns of  type $S_3$.
This observation finishes the proof. \qed

{\bf Proof of Proposition \ref{prop5}:}
Assume that $\nabla$ is an ES basic fusion operator and let $\Phi\mapsto\succeq_\Phi$ be the basic assignment associated to $\nabla$ by Theorem \ref{teo1}. Let $i$ be an agent and $E_i$ be  an $i$-profile and $w$, $w'$, $w''$ be three different interpretations in $\mathcal{W_P}$.

\begin{enumerate}
  \item[(i)] If \mbox{$w'\succeq_{E_i} w''$}, by \sintSD, given $j$ in $\mathcal{S}$ there exists a $j$-profile $E_j$
  satisfying that \mbox{$B\big(\nabla(E_j,E_{w,w'})\big)\equiv\varphi_{w}$} and
  $B\big(\nabla(E_j,E_{w',w''})\big)\equiv B\big(\nabla(E_i,E_{w',w''})\big)$
  (the last equivalence uses (iii) of \sintSD\ if \mbox{$w'\simeq_{E_i} w''$} and uses (iv)
  of \sintSD\ if \mbox{$w'\succ_{E_i} w''$}). From this and \eqref{B-Rep}, we get  \mbox{$w\succ_{E_j}w'$}.
  Moreover, since \mbox{$w'\succeq_{E_i}w''$} and \linebreak $B\big(\nabla(E_j,E_{w',w''})\big)\equiv B\big(\nabla(E_i,E_{w',w''})\big)$, by \eqref{B-Rep},
  we have that \mbox{$w'\succeq_{E_j}w''$}. Thus, by transitivity,  \mbox{$w\succ_{E_j}w''$}.
  Therefore, by \eqref{B-Rep}, $B\big(\nabla(E_j,E_{w,w''})\big)\equiv\varphi_{w}$.
  If  \mbox{$w''\succeq_{E_i} w'$}, the proof is analogous, interchanging the roles of $w'$ and $w''$.

 \item[(ii)] If \mbox{$w\succeq_{E_i} w''$}, by \sintSD, given $k$ in $\mathcal{S}$ there exists a $k$-profile $E_k$ satisfying that \mbox{$B\big(\nabla(E_k,E_{w',w''})\big)\equiv\varphi_{w''}$} and $B\big(\nabla(E_k,E_{w,w''})\big)\equiv B\big(\nabla(E_i,E_{w,w''})\big)$
     (the last equivalence uses (iii) of \sintSD\ if $w\simeq_{E_i}w''$  and uses (iv)  of \sintSD\ if $w\succ_{E_i} w''$). Thus, since $w\succeq_{E_i}w''$, $B\big(\nabla(E_k,E_{w,w''})\big)\equiv B\big(\nabla(E_i,E_{w,w''})\big)$ and $B\big(\nabla(E_k,E_{w',w''})\big)\equiv\varphi_{w''}$, from \eqref{B-Rep}, we have that \mbox{$w\succeq_{E_k}w''$} and \mbox{$w''\succ_{E_k}w'$}. Then \mbox{$w\succ_{E_k} w'$}, that is, by \eqref{B-Rep}, $B\big(\nabla(E_k,E_{w,w'})\big)\equiv\varphi_{w}$. If $w''\succeq_{E} w$, the proof is analogous, interchanging the roles of $w$ and $w''$.\qed
\end{enumerate}

{\bf Proof of Proposition \ref{prop6}:}
Suppose that $\nabla$ satisfies \sintii,  \sintvii\ and \sintviii. In order to  show \sintU, we proceed by induction on the length of finite societies: we will show that for every $N$ in $\mathcal{F}^\ast(\mathcal{S})$, $N$ satisfies the following property:
\begin{description}
\item[({\bf U})] For all $N$-profile $\Phi$ and every $E$ in $\mathcal{E}$, if $E_i=E_j$, for all $i$, $j$ in $N$, then $B\big(\nabla(\Phi,E)\big)\equiv B\big(\nabla(E_i,E)\big)$, for all $i$ in $N$.
\end{description}

If $N$ has only one agent, {\bf (U)} holds trivially. Now, let $n$ be a positive integer, with $n>1$, and suppose that for all $D$ in $\mathcal{F}^\ast({\mathcal{S}})$, formed by a number of agents less than $n$, it satisfies {\bf (U)}.  Consider $N$ in $\mathcal{F}^\ast(\mathcal{S})$ with $n$ agents and let us see that $N$ also satisfies {\bf (U)}.
 Consider  an $N$-profile $\Phi$ such that $E_i=E_j$, for every couple $i$, $j$ in $N$, and let $E$ be an epistemic state. Consider $D=N\setminus\{k\}$, for some $k$ in $N$. Thus, if $i$ is any agent in $D$,  from {\bf (U)} we have that $B\big(\nabla(\Phi\upharpoonright_D,E)\big)\equiv
B\big(\nabla(E_i,E)\big)$. Moreover, since $E_k=E_i$, by \sintii\ we have that  $B\big(\nabla(E_k,E)\big)\equiv B\big(\nabla(E_i,E)\big)$. Thus,
$B\big(\nabla(\Phi\upharpoonright_D,E)\big)\wedge B\big(\nabla(E_k,E)\big)\equiv B\big(\nabla(E_i,E)\big)$, for all $i$ in $N$. From the last equivalence, the result follows using \sintvii\ and \sintviii. \qed

{\bf Proof of Proposition \ref{prop7}:}
Let us suppose $N$ in $\mathcal{F}^\ast(\mathcal{S})$ and an $N$-profile $\Phi$ such that $E_i=E_j$ for any couple $i$, $j$ in $N$.
\begin{description}
    \item[(i)$\Rightarrow$ (ii)] Trivial.

    \item[(ii)$\Rightarrow$ (iii)] Let $w$, $w'$ be a pair of interpretations and consider $E$ an epistemic state such that $\model{B(E)}=\{w,w'\}$. Since $E_i=E_j$, for any pair of agents $i$, $j$ in $N$,  from (ii) we have that $B\big(\nabla(\Phi,E)\big)\equiv B\big(\nabla(E_i,E)\big)$, for all $i$ in $N$. Thus, by \eqref{B-Rep}, we have $w\succeq_\Phi w'$ {\it iff} $w\succeq_{E_i} w'$, for all $i$ in $N$.

    \item[(iii)$\Rightarrow$ (i)] Let $E$ be an epistemic state. Since $E_i=E_j$, for any pair of agents $i$, $j$ in $N$, by Property ({\it u}) we have $\succeq_\Phi=\succeq_{E_i}$ for all $i$ in $N$. From this, by \eqref{B-Rep},  $B\big(\nabla(\Phi,E)\big)\equiv B\big(\nabla(E_i,E)\big)$, for all $i$ in $N$.\qed
    \end{description}

{\bf Proof of Proposition \ref{prop16}:}
Let $\nabla$ be an ES combination operator, and suppose that it satisfies \sintvii\ and \sintviiiw. In order  to show that \sintP\ holds, we will proceed by induction on the length of finite societies. Indeed, we will show that, for all $N$ in $\mathcal{F}^\ast(\mathcal{S})$, $N$ satisfies the following:
\begin{itemize}
  \item[\bf (P)] For any $N$-profile $\Phi$, any pair of epistemic states $E$, $E'$ in $\mathcal{E}$, if $B\big(\nabla(E_i,E)\big)\wedge B(E')\vdash\bot$ for all $i$ in $N$, and $\Bigwedge_{i\in N} B\big(\nabla(E_i,E)\big)\not\vdash\bot$, then $B\big(\nabla(\Phi,E)\big)\wedge B(E')\vdash\bot$.
\end{itemize}

If $N$ has a sole agent, the result is trivial. Now, assume $n>1$ a positive integer and suppose that, for all
$D$ in $\mathcal{F}^\ast(\mathcal{S})$ with less than $n$ agents, $D$ satisfies {\bf(P)}.
Consider $N$ in $\mathcal{F}^\ast(\mathcal{S})$ with $n$ agents and let us show that $N$ also satisfies {\bf(P)}.
Let $\Phi$ be an $N$-profile and $E$, $E'$ a pair of epistemic states such that $B\big(\nabla(E_i,E)\big)\wedge B(E')\vdash\bot$ for all $i$ in $N$
and $\Bigwedge_{i\in N} B\big(\nabla(E_i,E)\big)\not\vdash\bot$.
Let $j$ be an agent in $N$ and let us consider $D=N\setminus\{j\}$.
Thus, $B\big(\nabla(E_i,E)\big)\wedge B(E')\vdash\bot$ for all $i$ in $D$ and $\Bigwedge_{i\in D} B\big(\nabla(E_i,E)\big)\not\vdash\bot$.
From this  and the induction hypothesis (\ie\
{\bf(P)} for the profile $\Phi\upharpoonright_{_D}$), we have that $B\big(\nabla(\Phi\upharpoonright_{_D},E)\big)\wedge B(E')\vdash\bot$.
Moreover, by \sintvii, it is easy to see that $\Bigwedge_{i\in D} B\big(\nabla(E_i,E)\big)\vdash B\big(\nabla(\Phi\upharpoonright_{_D},E)\big)$.
Then, $\Bigwedge_{i\in N} B\big(\nabla(E_i,E)\big)\vdash B\big(\nabla(\Phi\upharpoonright_{_D},E)\big)\wedge B\big(\nabla(E_j,E)\big)$,
and since $\Bigwedge_{i\in N} B\big(\nabla(E_i,E)\big)\not\vdash\bot$
then  \linebreak $B\big(\nabla(\Phi\upharpoonright_{_D},E)\big)\wedge B\big(\nabla(E_j,E)\big)\not\vdash\bot$.
Therefore, by \sintviiiw, we have that $B\big(\nabla(\Phi,E)\big)\vdash B\big(\nabla(\Phi\upharpoonright_{_D},E)\big)\vee B\big(\nabla(E_j,E)\big)$.
From this and the fact that $\big[B\big(\nabla(\Phi\upharpoonright_{_D},E)\big)\vee B\big(\nabla(E_j,E)\big)\big]\wedge B(E')\vdash\bot$,
we get $B\big(\nabla(\Phi,E)\big)\wedge B(E')\vdash\bot$, as desired.\qed

{\bf Proof of Proposition \ref{prop8}:}
Let $N$ be in $\mathcal{F}^\ast(\mathcal{S})$, $\Phi$ be  an $N$-profile and $E$, $E'$ be a pair of epistemic states.

\begin{description}
    \item[(i)$\Rightarrow$ (ii)] Suppose that $B(E)$ has at  most two models and assume that $B\big(\nabla(E_i,E)\big)\wedge B (E')\vdash\bot$, for all $i$ in $N$.
        We want to show that $B\big(\nabla(\Phi,E)\big)\wedge B (E')\vdash\bot$.
        If $B(E)\wedge B(E')\vdash\bot$, the result follows straightforwardly from \sinti. Now suppose that $B(E)\wedge B(E')\not\vdash\bot$ and  let $w'$ be a model of $B(E)\wedge B(E')$. If $w'$ were the sole model of $B(E)$, then, by \sinti, $w'$ were also a model of $B\big(\nabla(E_i,E)\big)$, for each $i$ in $N$.  Thus, if we consider $j$ in $N$, we have that $B\big(\nabla(E_j,E)\big)\wedge B(E')\not\vdash\bot$, a contradiction.
        Thus, $B(E)$ has two models. Let
         $w$ be the other model of $B(E)$, in particular $w\neq w'$. Since $B\big(\nabla(E_i,E)\big)\wedge B(E')\vdash\bot$,  we have, by \sinti, that   $w$ is the sole model of $B\big(\nabla(E_i,E)\big)$, for all $i$ in $N$. Therefore, $\Bigwedge_{i\in N} B\big(\nabla(E_i,E)\big)\not\vdash\bot$. Then, by  \sintP, $B\big(\nabla(\Phi,E)\big)\wedge B(E')\vdash\bot$.

    \item[(ii)$\Rightarrow$ (iii)] Let $w$, $w'$ be a pair of interpretations which satisfies $w\succ_{E_i} w'$, for all
        $i$ in $N$, and assume that $E$, $E'$ are epistemic states which satisfy $\model{B(E)}=\{w,w'\}$ and $\model{B(E')}=\{w'\}$. Thus, by \eqref{B-Rep} we have that for all $i$ in $N$,  $B\big(\nabla(E_i,E)\big)\wedge B(E')\vdash\bot$. From this and the assumption that (ii)
        holds, we get $B\big(\nabla(\Phi,E)\big)\wedge
        B(E')\vdash\bot$, and then, necessarily, $\model{B\big(\nabla(\Phi,E)\big)}=\{w\}$. By \eqref{B-Rep}, this last equality entails $w\succ_\Phi w'$.

    \item[(iii)$\Rightarrow$ (i)] Let us suppose that $B\big(\nabla(E_i,E)\big)\wedge B (E')\vdash\bot$, for all $i$ in $N$,  and $\Bigwedge_{i\in N} B\big(\nabla(E_i,E)\big)$ is consistent. Towards a contradiction, suppose that $B\big(\nabla(\Phi,E)\big)\wedge B(E')\not\vdash\bot$. Let $w$, $w'$ be  interpretations such that  $w\models\Bigwedge_{i\in N} B\big(\nabla(E_i,E)\big)$ and $w'\models B\big(\nabla(\Phi,E)\big)\wedge B(E')$. By \sinti, we have that both $w$ and $w'$ are models of $B(E)$. Moreover, note that if $i$ is an agent in $N$,  $w'\not\models B\big(\nabla(E_i,E)\big)$, because $w'\models B(E')$ and $B\big(\nabla(E_i,E)\big)\wedge B (E')\vdash\bot$.
        From this and \eqref{B-Rep}, it follows that $w\succ_{E_i}w'$, for all $i$ in $N$. Hence, by Property {\bf({\em p})} we have  $w\succ_\Phi w'$. From this and \eqref{B-Rep} again, it follows that $w'\not\models
        B\big(\nabla(\Phi,E)\big)$, a contradiction.\qed
    \end{description}

{\bf Proof of Proposition \ref{prop9}:}
Let $N$ be a finite society of agents in $\mathcal{S}$ and $\Phi$, $\Phi'$ be a pair of $N$-profiles.
    \begin{description}
    \item[(i)$\Rightarrow$ (ii)] Let $E$ be an epistemic state such that  $\big|B(E)\big|\leq 2$ and assume that $B\big(\nabla(E_i,E)\big)\equiv B\big(\nabla(E_i',E)\big)$, for each agent $i$ in $N$. We want to show that $B\big(\nabla(\Phi,E)\big)\equiv B\big(\nabla(\Phi',E)\big)$.
        By \sintI, it is enough to prove  that, for every epistemic state $E'$ in $\mathcal{E}$, if $B(E')\vdash B(E)$, then $B\big(\nabla(E_i,E')\big)\equiv B\big(\nabla(E_i',E')\big)$, for each agent $i$ in $N$. Thus, suppose that  $E'$ is an epistemic state such that $B(E')\vdash B(E)$ and let $i$ be any agent in $N$. If  $B(E')$ has exactly one model, the result follows directly from \sinti.
        Otherwise, $B(E')\equiv B(E)$. By \sintii\ we have that $B\big(\nabla(E_i,E)\big)\equiv B\big(\nabla(E_i,E')\big)$ and $B\big(\nabla(E_i',E)\big)\equiv B\big(\nabla(E_i',E')\big)$. From these equivalences and the fact that
        $B\big(\nabla(E_i,E)\big)\equiv B\big(\nabla(E_i',E)\big)$, we have
         $B\big(\nabla(E_i,E')\big)\equiv B\big(\nabla(E_i',E')\big)$.

    \item[(ii)$\Rightarrow$ (iii)] Assume that ({ii}) holds. Let  $w$, $w'$ be a pair of models and suppose that
        for each  $i$ in $N$, $\succeq_{E_i}\upharpoonright_{\{w,w'\}}=\succeq_{E_i'}\upharpoonright_{\{w,w'\}}$. Let $E$ be an epistemic state such that $\model{B(E)}=\{w,w'\}$. From \eqref{B-Rep}, we have
        $B\big(\nabla(E_i,E)\big)\equiv B\big(\nabla(E_i',E)\big)$, for all $i$ in $N$. Hence, by the assumption that (ii) holds, $B\big(\nabla(\Phi,E)\big)\equiv B\big(\nabla(\Phi',E)\big)$. Finally, by \eqref{B-Rep},
        $\succeq_{\Phi}\upharpoonright_{\{w,w'\}}=\succeq_{\Phi'}\upharpoonright_{\{w,w'\}}.$

    \item[(iii)$\Rightarrow$ (i)] Let $E$ be an epistemic state. Suppose that for every epistemic state $E'$, with \mbox{$B(E')\vdash B(E)$},
    we have \linebreak $B\big(\nabla(E_i,E')\big)\equiv B\big(\nabla(E_i',E')\big)$, for each  agent $i$ in $N$.
    We have to show that $B\big(\nabla(\Phi,E)\big)\equiv B\big(\nabla(\Phi',E)\big)$.
    In order to do this, it is enough to prove that $B\big(\nabla(\Phi,E)\big)\vdash B\big(\nabla(\Phi',E)\big)$
    (the proof of the converse is similar). Let $w$ be a model of $B\big(\nabla(\Phi,E)\big)$ and, towards a contradiction,
    suppose that $w\not\models B\big(\nabla(\Phi',E)\big)$. Let $w'$ be a model of $B\big(\nabla(\Phi',E)\big)$ and  $E'$ be an epistemic state with
        $\model{B(E')}=\{w,w'\}$. Thus, $B(E')\vdash B(E)$. Then, for any agent $i$  in $N$
        we have $B\big(\nabla(E_i,E')\big)\equiv B\big(\nabla(E_i',E')\big)$. Then, by \eqref{B-Rep},
        $\succeq_{E_i}\upharpoonright_{\{w,w'\}}=\succeq_{E_i'}\upharpoonright_{\{w,w'\}}$,
        for all $i$ in $N$. Therefore, from Property {\bf ({\bf \it ind})}
        it follows that $\succeq_{_\Phi}\upharpoonright_{_{\{w,w'\}}}=\succeq_{_{\Phi'}}\upharpoonright_{_{\{w,w'\}}}$.
        From this and \eqref{B-Rep}, we have $B\big(\nabla(\Phi,E')\big)\equiv B\big(\nabla(\Phi',E')\big)$.
        Moreover, since $B(E')\equiv B(E)\wedge B(E')$, $B\big(\nabla(\Phi,E)\big)\wedge B(E')\not\vdash\bot$     and
        $B\big(\nabla(\Phi',E)\big)\wedge B(E')\not\vdash\bot$, by \sintiii\ and \sintiv, we have
        $B\big(\nabla(\Phi,E')\big)\equiv B\big(\nabla(\Phi,E)\big)\wedge B(E')$ and
        $B\big(\nabla(\Phi',E')\big)\equiv B\big(\nabla(\Phi',E)\big)\wedge B(E')$.
        Note that, since $w$ is a model of $B\big(\nabla(\Phi,E)\big)\wedge B(E')$, then using the last two equivalences,
        we get that $w$ is also a model of $B\big(\nabla(\Phi',E)\big)\wedge B(E')$, a contradiction.\qed
\end{description}

{\bf Proof of Proposition \ref{prop10}:}
Let $\nabla$ be an ES basic fusion operator which satisfies \sintU\ and \sintI. We want to show that it satisfies \sintP.
By Proposition \ref{prop8}, it is enough to prove that Property (ii) of this proposition holds.
In order to do this,  take $N$ in $\mathcal{F}^\ast(\mathcal{S})$, $\Phi$ an $N$-profile ,
and let $E$, $E'$ be a pair of epistemic states, such that $B(E)$ has at most two models.
Suppose that, for all $i$ in $N$, $B\big(\nabla(E_i,E)\big)\wedge B(E')\vdash\bot$.
Towards a contradiction, suppose that $B\big(\nabla(\Phi,E)\big)\wedge B(E')\not\vdash\bot$.
Let $w'$ be a model of $B\big(\nabla(\Phi,E)\big)\wedge B(E')$. Note that for every $i\in N$,
since \linebreak$B\big(\nabla(E_i,E)\big)\wedge B(E')\vdash\bot$, by \sinti,
there exists $w$, a model of $B(E)$, which is, necessarily, the  unique model of $B\big(\nabla(E_i,E)\big)$.
Let us fix an agent $j$ in $N$ and consider the $N$-profile $\Phi'$, the entries of which  are all equal to $E_j$,
\ie\ for all $i\in N$, $E_i'=E_j$.
Thus, $B\big(\nabla(E_i,E)\big)\equiv B\big(\nabla(E_i',E)\big)$ for all $i$ in $N$.
Thus, by \sintI\ and Proposition \ref{prop9}, we have  $B\big(\nabla(\Phi,E)\big)\equiv B\big(\nabla(\Phi',E)\big).$ Moreover, by  \sintU,
 $B\big(\nabla(\Phi',E)\big)\equiv B\big(\nabla(E_j,E)\big)$.
 Hence, $B\big(\nabla(\Phi,E)\big)\wedge B(E')\equiv B\big(\nabla(E_j,E)\big)\wedge B(E')$,
 and then $B\big(\nabla(E_j,E)\big)\wedge B(E')\not\vdash\bot$, a contradiction.\qed

{\bf Proof of Proposition \ref{prop11}:}
Let $N$ be an element of $\mathcal{F}^\ast(\mathcal{S})$.
\begin{description}
    \item[(i)$\Rightarrow$ (ii)] Straightforward.

    \item[(ii)$\Rightarrow$ (iii)] Assume that $d_N$ is the agent in $N$ which satisfies ({\it ii}). Consider $\Phi$ an $N$-profile and $w$, $w'$ a pair of interpretations such that $w\succ_{E_{d_N}} w'$. Let $E$ be an epistemic state such that $\model{B(E)}=\{w,w'\}$. Note that, by \eqref{B-Rep}, $w$ is the sole model of $B\big(\nabla(E_{d_N},E)\big)$. Thus, since $B\big(\nabla(\Phi,E)\big)\vdash B\big(\nabla(E_{d_N},E)\big)$, we have that $w$ is also the sole model of $B\big(\nabla(\Phi,E)\big)$. Therefore, from \eqref{B-Rep}, it follows that $w\succ_\Phi w'$.

    \item[(iii)$\Rightarrow$ (i)] Suppose that $\Phi\mapsto\succeq_\Phi$ satisfies {\bf({\em d})}, and let $d_N$  be the agent in $N$ satisfying Property {\bf({\em d})}. Let  $\Phi$ be an \linebreak$N$-profile,  $E$ be an epistemic state. Suppose, towards a contradiction, that $B\big(\nabla(\Phi,E)\big)\not\vdash B\big(\nabla(E_{d_N},E)\big)$. Thus, let us choose $w'$, a model of $B\big(\nabla(\Phi,E)\big)$, such that $w'\not\models B\big(\nabla(E_{d_N},E)\big)$. Then, by \eqref{B-Rep}, there exists $w$, a model of $B(E)$, such that $w\succ_{E_{d_N}} w'$. Then, by Property {\bf({\em d})}, $w\succ_\Phi w'$. By
        \eqref{B-Rep} again, we have $w'\not\models B\big(\nabla(\Phi,E)\big)$, a contradiction.\qed
    \end{description}

{\bf Proof of Proposition \ref{prop12}:}
Let $\nabla$ be an ES combination operator that satisfies \sintD. Let $N$ be a finite society of agents in $\mathcal{S}$,  $\Phi$ be an $N$-profile and $E$,  $E'$ be a pair of epistemic states in $\mathcal{S}$ such that $B\big(\nabla(E_i,E)\big)\wedge B(E')\vdash\bot$ for all $i$ in $N$, and $\Bigwedge_{i\in N} B\big(\nabla(E_{i},E)\big)\not\vdash\bot$. By \sintD,  there exists  $d_N$ in $N$ with $B\big(\nabla(\Phi,E)\big)\vdash B\big(\nabla(E_{d_N},E)\big)$. Thus, $B\big(\nabla(\Phi,E)\big)\wedge B(E')\vdash B\big(\nabla(E_{d_N},E)\big)\wedge B(E')$. Therefore $B\big(\nabla(\Phi,E)\big)\wedge B(E')\vdash\bot$,
as desired. \qed

{\bf Proof of Proposition \ref{propnew-dic}:}
Suppose that $D=\{d\}$ for an agent $d$ in $N$. Let us see that $d$ is an $N$-dictator, with respect to $\nabla$.
In order to see this, consider  an $N$-profile $\Phi$ and an epistemic state $E$. Towards a contradiction,
 suppose that $B\big(\nabla(\Phi,E)\big)\not\vdash B\big(\nabla(E_d,E)\big)$. Let $w$ be a model of
$B\big(\nabla(\Phi,E)\big)$ such that $w\not\models B\big(\nabla(E_d,E)\big)$, and consider $E_{w}$ an epistemic state
with $\model{B(E_w)}=\{w\}$. Thus, $B\big(\nabla(E_d,E)\big)\wedge B(E_w)\vdash\bot$, and since $\{d\}$ is decisive,
we have that $B\big(\nabla(\Phi,E)\big)\wedge B(E_w)\vdash\bot$. Therefore, $w\not\models B\big(\nabla(\Phi,E)\big)$, a contradiction.\qed

{\bf Proof of Proposition \ref{prop14}:}
Assume that ${E_{w,w'}}D^\nabla{E_{w'}}$, for a pair of epistemic states $E_{w,w'}$ and $E_{w'}$ with \linebreak$\model{B(E_{w,w'})}=\{w,w'\}$ and $\model{B(E_{w'})}=\{w'\}$. We want to show that, if $w''$ is an  interpretation  different from $w$  and $w'$, we have ${E_{w,w''}}D^{\nabla^\ast}{E_{w''}}$, for any epistemic states $E_{w,w''}$ and $E_{w''}$  such that $\model{B(E_{w,w''})}=\{w,w''\}$ and $\model{B(E_{w''})}=\{w''\}$.  In order to do this,  suppose that $\Phi$ is  an $N$-profile such that $B\big(\nabla(E_i,E_{w,w''})\big)\wedge B(E_{w''})\vdash\bot$, for all $i$ in $D$, and consider $E_{w',w''}$ an epistemic state such that $\model{B(E_{w',w''})}=\{w',w''\}$. Thus, since $\nabla$ satisfies \sintSD, by Proposition \ref{prop5}, there exists  an $N$-profile  $\Phi'$ which satisfies the following:

\begin{enumerate}
  \item[(i)] For all $i$ in $D$:
    \begin{itemize}
        \item $B\big(\nabla(E_i',E_{w,w'})\big)\equiv\varphi_w$, and
        \item $B\big(\nabla(E_i',E_{w',w''})\big)\equiv\varphi_{w'}$
    \end{itemize}

  \item[(ii)] For all $j$ in $N\setminus D$:
    \begin{itemize}
        \item $B\big(\nabla(E_j',E_{w,w'})\big)\equiv B\big(\nabla(E_j',E_{w',w''})\big)\equiv\varphi_{w'}$, and
        \item $B\big(\nabla(E_j',E_{w,w''})\big)\equiv B\big(\nabla(E_j,E_{w,w''})\big)$
    \end{itemize}
\end{enumerate}

Since $B\big(\nabla(E_i',E_{w,w'})\big)\equiv\varphi_w$ for all $i$ in $D$, and $B\big(\nabla(E_j',E_{w,w'})\big)\equiv\varphi_{w'}$ for all $j$ in
$N\setminus D$, then:
\begin{itemize}
  \item $B\big(\nabla(E_i',E_{w,w'})\big)\wedge B(E_{w'})\vdash\bot$, for all $i$ in $D$
  \item $B\big(\nabla(E_j',E_{w,w'})\big)\equiv B(E_{w'})$, for all $j$ in $N\setminus D$
  \item  $\Bigwedge_{i\in D} B\big(\nabla(E_i',E_{w,w'})\big)\not\vdash\bot$
\end{itemize}
Then, because of ${E_{w,w'}}D^\nabla{E_{w'}}$, we have that $B\big(\nabla(\Phi',E_{w,w'})\big)\wedge B(E_{w'})\vdash\bot$. Thus, by \sinti, we have \begin{equation}\label{eq:12}
B\big(\nabla(\Phi',E_{w,w'})\big)\equiv\varphi_{w}
\end{equation}
Moreover, for all $i$  in $N$, $B\big(\nabla(E_i',E_{w',w''})\big)\equiv\varphi_{w'}$, thus $B\big(\nabla(E_i',E_{w',w''})\big)\wedge B(E_{w''})\vdash\bot$. Then, by \sintP, we have  $B\big(\nabla(\Phi',E_{w',w''})\big)\wedge B(E_{w''})\vdash\bot$. From this and \sinti\ we get
\begin{equation}\label{eq:13}
  B\big(\nabla(\Phi',E_{w',w''})\big)\equiv\varphi_{w'}
\end{equation}
Thus, by the equivalences \eqref{eq:12} and \eqref{eq:13}, from \eqref{B-Rep} and the transitivity  of $\succeq_{\Phi'}$ it follows that
\begin{equation}\label{eq:14}
B\big(\nabla(\Phi',E_{w,w''})\big)\equiv \varphi_{w}
\end{equation}

Now, given $i$ in $D$, since $B\big(\nabla(E_i,E_{w,w''})\big)\wedge B(E_{w''})\vdash\bot$, by \sinti\ we have that $B\big(\nabla(E_i,E_{w,w''})\big)\equiv\varphi_w$. Moreover, $B\big(\nabla(E_i',E_{w,w'})\big)\equiv\varphi_w$ and
$B\big(\nabla(E_i',E_{w',w''})\big)\equiv\varphi_{w'}$ together imply $B\big(\nabla(E_i',E_{w,w''})\big)\equiv\varphi_w$. Therefore, for all $i$ in $D$,  $B\big(\nabla(E_i,E_{w,w''})\big)\equiv B\big(\nabla(E_i',E_{w,w''})\big)$. From this  and the fact that
$B\big(\nabla(E_j',E_{w,w''})\big)\equiv B\big(\nabla(E_j,E_{w,w''})\big)$ for all $j$ in $N\setminus D$, it follows that
$B\big(\nabla(E_i,E_{w,w''})\big)\equiv B\big(\nabla(E_i',E_{w,w''})\big)$ for all $i$ in $N$. Hence, by \sintI\ and Proposition \ref{prop9}, we have that $B\big(\nabla(\Phi,E_{w,w''})\big)\equiv B\big(\nabla(\Phi',E_{w,w''})\big)$. From this and Equivalence \eqref{eq:14}, it follows $B\big(\nabla(\Phi,E_{w,w''})\big)\wedge B(E_{w''})\vdash\bot$.\qed

{\bf Proof of Proposition \ref{prop15}:}
Let $w''$ be an interpretation, with $w''\not\models B(E_{w,w'})$, and $E_{w',w''}$ an epistemic state such that $\model{B(E_{w',w''})}=\{w',w''\}$. Suppose that ${E_{w,w'}}D^{\nabla}{E_{w'}}$ and let us see that ${E_{w',w''}}D^{\nabla^\ast}{E_{w'}}$. In order to do that, take  an \linebreak$N$-profile $\Phi$ such that $B\big(\nabla(E_i,E_{w',w''})\big)\wedge B(E_{w'})\vdash\bot$, for all $i$ in $D$. Let $E_{w,w''}$ be an epistemic state such that $\model{B(E_{w,w''})}=\{w,w''\}$. By \sintSD\ and Proposition \ref{prop5}, there exists an $N$-profile $\Phi'$  which satisfies the following:
\begin{enumerate}
  \item[(i)] $B\big(\nabla(E_i',E_{w,w'})\big)\equiv\varphi_w$ and $B\big(\nabla(E_i',E_{w,w''})\big)\equiv\varphi_{w''}$, for all $i$ in $D$.
  \item[(ii)] $B\big(\nabla(E_j',E_{w,w'})\big)\equiv\varphi_{w'}$, $B\big(\nabla(E_j',E_{w,w''})\big)\equiv\varphi_{w''}$ and
      $B\big(\nabla(E_j',E_{w',w''})\big)\equiv B\big(\nabla(E_j,E_{w',w''})\big)$,
      for all $j$ in $N\setminus D$.
\end{enumerate}
From these statements it follows that:
\begin{itemize}
  \item $B\big(\nabla(E_i',E_{w,w'})\big)\wedge B(E_{w'})\vdash\bot$, for all $i$ in $D$
  \item $B\big(\nabla(E_j',E_{w,w'})\big)\equiv B(E_{w'})$, for all $j$ in $N\setminus D$
  \item  $\Bigwedge_{i\in D} B\big(\nabla(E_i',E_{w,w'})\big)\not\vdash\bot$
\end{itemize}
Thus, since $_{E_{w,w'}}D_{E_{w'}}^{\nabla}$, we have  $B\big(\nabla(\Phi',E_{w,w'})\big)\wedge B(E_{w'})\vdash\bot$. Hence, by \sinti, it follows
\begin{equation}\label{eq:15}
B\big(\nabla(\Phi',E_{w,w'})\big)\equiv\varphi_{w}
\end{equation}
Note that $B\big(\nabla(E_i',E_{w,w''})\big)\wedge
B(E_{w})\vdash\bot$, for all $i$ in $N$. Hence, by \sintP, we have $B\big(\nabla(\Phi',E_{w,w''})\big)\wedge B(E_{w})\vdash\bot$. Thus, by \sinti, we have:
\begin{equation}\label{eq:16}
  B\big(\nabla(\Phi',E_{w,w''})\big)\equiv\varphi_{w''}
\end{equation}
From the equivalences  \eqref{eq:15} and \eqref{eq:16}, \eqref{B-Rep} and transitivity of $\succeq_{\Phi'}$, it follows that $B\big(\nabla(\Phi',E_{w',w''})\big)\equiv\varphi_{w''}$. Now, if $i$ is an agent in $D$, since $B\big(\nabla(E_i,E_{w',w''})\big)\wedge B(E_{w'})\vdash\bot$, \sinti\ implies that $B\big(\nabla(E_i,E_{w',w''})\big)\equiv\varphi_{w''}$.
   Moreover, $B\big(\nabla(E_i',E_{w,w'})\big)\equiv\varphi_w$ and $B\big(\nabla(E_i',E_{w,w''})\big)\equiv\varphi_{w''}$ together entail $B\big(\nabla(E_i',E_{w',w''})\big)\equiv\varphi_{w''}$.
   Thus, for all $i$ in $D$,
$B\big(\nabla(E_i,E_{w',w''})\big)\equiv B\big(\nabla(E_i',E_{w',w''})\big)$. Moreover,  $B\big(\nabla(E_i,E_{w',w''})\big)\equiv B\big(\nabla(E_i',E_{w',w''})\big)$, for all $i$ in $N$. Therefore, by \sintI, we have that $B\big(\nabla(\Phi,E_{w',w''})\big)\equiv B\big(\nabla(\Phi',E_{w',w''})\big)$. Then, by the equivalence \eqref{eq:16}, it follows that $B\big(\nabla(\Phi,E_{w',w''})\big)\wedge
B(E_{w'})\vdash\bot$.\qed

{\bf Proof of Theorem \ref{teo5}:}
Consider $w$, $w'$ a pair of interpretations and assume that $E_{w,w'}$ and $E_{w'}$ are epistemic states which satisfy
$\model{B(E_{w,w'})}=\{w,w'\}$,
$\model{B(E_{w'})}=\{w'\}$ and  ${E_{w,w'}}D^\nabla{E_{w'}}$. We want to show that $D$ is decisive, with respect to $\nabla$.
In order to do this, by Proposition \ref{prop13}, it is enough to show that for each pair of epistemic states, $E$, $E'$, such that $B(E)$
has at  most two models and $B(E')$ has  exactly one model, we have  that $ED^{\nabla^\ast}E'$. First, let $w''$ be an
interpretation different from $w$ and $w'$. We prove that
 for any pair of different interpretations  $w_1$, $w_2$ in
$\{w,w',w''\}$ we have that $E_{w_1,w_2}D^{\nabla^\ast}E_{w_2}$. For that purpose, we note the following:
\begin{itemize}
  \item Since $E_{w,w'}D^{\nabla}E_{w'}$, by Proposition \ref{prop14}, it follows  $E_{w,w''}D^{\nabla^\ast}E_{w''}$.

  \item Since $E_{w,w'}D^\nabla E_{w'}$, by Proposition \ref{prop15}, we have  $E_{w',w''}D^{\nabla^\ast}E_{w'}$.

  \item Since $E_{w,w''}D^{\nabla^\ast}E_{w''}$, by Observation \ref{obs1}, it follows $E_{w,w''}D^\nabla E_{w''}$.
  Thus, by Proposition \ref{prop14}, we have  $E_{w,w'}D^{\nabla^\ast}E_{w'}$.

  \item Since $E_{w',w''}D^{\nabla^\ast}E_{w'}$, by Observation \ref{obs1}, it follows $E_{w',w''}D^\nabla E_{w'}$. Thus, by Proposition \ref{prop14}, we have $E_{w,w''}D^{\nabla^\ast}E_{w}$.

  \item Since $E_{w,w''}D^\nabla E_{w''}$, by Proposition \ref{prop15}, it follows $E_{w',w''}D^{\nabla^\ast}E_{w''}$.

  \item Since $E_{w',w''}D^{\nabla^\ast}E_{w''}$, by Observation \ref{obs1}, we have $E_{w',w''}D^{\nabla} E_{w''}$. Then, by Proposition \ref{prop14}, we have  $E_{w,w'}D^{\nabla^\ast}E_{w}$.

\end{itemize}

Now, let $w_1$, $w_2$ be a pair of different interpretations. If either $w_1$ or $w_2$, is in $\{w,w'\}$, from the above statements we have $E_{w_1,w_2}D^{\nabla^\ast}E_{w_2}$. Hence, we suppose that both $w_1$ and $w_2$ are not elements of $\{w,w'\}$. Since $E_{w,w'}D^\nabla E_{w'}$, by Proposition \ref{prop15}, we have  $E_{w_1,w'}D^{\nabla^\ast}E_{w'}$, and thus $E_{w_1,w'}D^\nabla E_{w'}$. From this, by  Proposition \ref{prop14}, $E_{w_1,w_2}D^{\nabla^\ast}E_{w_2}$.\qed

{\bf Proof of Theorem \ref{teo6}:}
Let $N$ be a finite society of agents in $\mathcal{S}$. Define the following set:
$$\mathfrak{X}=\{X\subseteq N: X \mbox{ is decisive with respect to} \nabla\}.$$
 Since we have \sintP, by Observation \ref{obs1}, $N$ is decisive. Thus, $\mathfrak{X}$ is non empty. Let  $D$ be an element of  $\mathfrak{X}$ having minimal cardinality. Since  $D$ is decisive then $D\neq\emptyset$, as we noted in Observation \ref{obs1}. Thus, by the same observation, it is enough to see that $D$ has a single model in order to prove that $\nabla$ is dictatorial. Towards a contradiction, suppose that $D$ has at least two agents. Let $i$ be an agent in $D$.
Take three different interpretations $w$, $w'$, $w''$. Let $E_{w,w'}$, $E_{w,w''}$, $E_{w',w''}$, $E_{w}$ and $E_{w''}$ be epistemic states
such that  $\model{B(E_{w,w'})}=\{w,w'\}$, $\model{B(E_{w,w''})}=\{w,w''\}$, $\model{B(E_{w',w''})}=\{w',w''\}$,
$\model{B(E_{w})}=\{w\}$ and $\model{B(E_{w''})}=\{w''\}$.
By \sintSD, there exists an $N$-profile $\Phi$ which satisfies the following:
 \begin{itemize}
   \item $B\big(\nabla(E_i,E_{w,w'})\big)\equiv \varphi_w$ and $B\big(\nabla(E_i,E_{w',w''})\big)\equiv \varphi_{w'}$.
   \item $B\big(\nabla(E_j,E_{w,w''})\big)\equiv \varphi_{w''}$ and $B\big(\nabla(E_j,E_{w',w''})\big)\equiv \varphi_{w'}$, for all $j$ in $D\setminus\{i\}$.
   \item $B\big(\nabla(E_k,E_{w,w'})\big)\equiv \varphi_{w}$ and $B\big(\nabla(E_k,E_{w,w''})\big)\equiv \varphi_{w''}$, for all $k$ in $N\setminus D$.
 \end{itemize}
Note that, if $j$ is an agent in $D$, then $B\big(\nabla(E_j,E_{w',w''})\big)\wedge B(E_{w''})\vdash\bot$. Thus, since  $D$ is decisive, we have in particular $E_{w',w''}D^{\nabla^\ast} E_{w''}$. Therefore
\begin{equation}\label{eq17}
  B\big(\nabla(\Phi,E_{w',w''})\big)\wedge B(E_{w''})\vdash\bot
\end{equation}
We claim that $w\models B\big(\nabla(\Phi,E_{w,w'})\big)$.
Towards a contradiction, suppose that $B\big(\nabla(\Phi,E_{w,w'})\big)\wedge B(E_{w})\vdash\bot$. By Proposition \ref{prop5}, there exists   an $N$-profile $\Phi'$ such that
\begin{itemize}
  \item $B\big(\nabla(E_j',E_{w,w'})\big)\wedge B(E_{w})\vdash\bot$, for all $j$ in $D\setminus\{i\}$
  \item $B\big(\nabla(E_k',E_{w,w'})\big)\equiv B(E_{w})$, for all $k$ in $[N\setminus D]\cup\{i\}$
\end{itemize}
By \sinti, we have that $\Bigwedge_{j\in D\setminus\{i\}} B\big(\nabla(E_j',E_{w,w'})\big)\not\vdash\bot$, actually $B\big(\nabla(E_j',E_{w,w'})\big)\equiv \phi_{w'}$, for all $j\in D\setminus\{i\}$. Moreover, for all $j$ in \mbox{$D\setminus\{i\}$}, we have  $B\big(\nabla(E_j,E_{w,w''})\big)\equiv \varphi_{w''}$ and
\mbox{$B\big(\nabla(E_j,E_{w',w''})\big)\equiv \varphi_{w'}$}, and so \mbox{$B\big(\nabla(E_j,E_{w,w'})\big)\equiv \varphi_{w'}$}.
Thus, $B\big(\nabla(E_j,E_{w,w'})\big)\equiv B\big(\nabla(E_j',E_{w,w'})\big)$, for all $j$ in $D\setminus\{i\}$. Hence, $B\big(\nabla(E_j,E_{w,w'})\big)\equiv B\big(\nabla(E_j',E_{w,w'})\big)$
for all $j$ in $N$.  From this, by \sintI\ and Proposition  \ref{prop9}, it follows that $B\big(\nabla(\Phi,E_{w,w'})\big)\equiv B\big(\nabla(\Phi',E_{w,w'})\big)$. Thus, $B\big(\nabla(\Phi',E_{w,w'})\big)\wedge B(E_{w})\vdash\bot$. Hence,  $E_{w,w'}D\setminus\{i\}^\nabla E_{w}$. Then, by Theorem \ref{teo5}, we have that $D\setminus\{i\}$ is decisive with respect to $\nabla$, a contradiction with respect to the assumption of minimality of $D$.

By the  statement \eqref{eq17}   and \sinti,
it follows that $B\big(\nabla(\Phi,E_{w',w''})\big)\equiv\varphi_{w'}$.
From this and the fact that \linebreak$w\models B\big(\nabla(\Phi,E_{w,w'})\big)$, by \eqref{B-Rep} and the transitivity of $\succeq_\Phi$,
we have that $B\big(\nabla(\Phi,E_{w,w''})\big)\equiv\varphi_{w}$. Therefore, we have $B\big(\nabla(\Phi,E_{w,w''})\big)\wedge B(E_{w''})\vdash\bot$.
By \sintSD, there exists  an $N$-profile $\Phi''$ satisfying
 $B\big(\nabla(E_i'',E_{w,w''})\big)\wedge B(E_{w''})\vdash\bot$ and $B\big(\nabla(E_j'',E_{w,w''})\big)\equiv B(E_{w''})$, for all $j$ in
$N\setminus\{i\}$. Hence, $B\big(\nabla(E_j,E_{w,w''})\big)\equiv B\big(\nabla(E_j'',E_{w,w''})\big)$, for all $j$ in $N$. Then, by \sintI,
$B\big(\nabla(\Phi,E_{w,w''})\big)\equiv B\big(\nabla(\Phi'',E_{w,w''})\big)$.
Thus, $B\big(\nabla(\Phi'',E_{w,w''})\big)\wedge B(E_{w''})\vdash\bot$, and therefore, $_{E_{w,w''}}\{i\}_{E_{w''}}^{\nabla}$.
Then, by Theorem \ref{teo5},  $\{i\}$ is  decisive, contradicting again the minimality of $D$.

Now we know that $D$ is decisive and has cardinality one. We conclude by Proposition \ref{propnew-dic}.\qed

\begin{proposition}\label{prop17}
 $\nabla^{\pi}$ is an ES basic fusion operator that satisfies \sintv, \sintvii, \sintviiiw, \sintSD, \sintU, \sintP, \sintI, \sintD,
 but does not satisfy \sintvi\ and \sintviii.
\end{proposition}

\begin{proof}
  From the definition of $\Phi\mapsto\succeq_{d_N}$,  we have straightforwardly that it is a basic assignment which satisfies structure preserving and therefore,  Property \assgi\ holds. Now consider $N$ a finite society of agents in  $\mathcal{S}$, and let $\Set{N_1, N_2}$ be a partition of $N$.  Thus, since  $\max(N)=\max\{\max(N_1),\max(N_2)\}$, if $w\succeq_{d_{N_1}}w'$ and $w\succeq_{d_{N_2}}w'$, then  $w\succeq_{d_N}w'$, that is, Property \assgiii\ holds. Similarly we prove that the projective assignment satisfies Property \assgivp.

To show that  Property \assgii\ does not hold, we consider a profile $\Phi=(\succeq_1,\succeq_2,\succeq_3)$ such that $\model{B(\succeq_1)}=\model{B(\succeq_2)}=\{00\}$ and $\model{B(\succeq_3)}=\{00,01\}$. Thus, $\model{\Bigwedge_{i\in N} B(\succeq_i)}=\{00\}$, and therefore $\model{\Bigwedge_{i\in N}B(\succeq_i)}\neq\model{B(\succeq_3)}$. To see that $\Phi\mapsto\succeq_{d_N}$ does not satisfy Property \assgiv,
 take the previous profile and
 consider $N_1=\{1,2\}$ and $N_2=\{3\}$. Since $00\succ_2 01$ and $00\simeq_3 01$, then $00\succ_\Phi\upharpoonright_{N_1}01$ and $00\simeq_\Phi\upharpoonright_{N_2}01$, but $00\simeq_\Phi 01$.

Thus, from  Proposition \ref{prop4}, it follows that $\nabla^\pi$ satisfies \sintv, \sintvii\ and \sintviiiw, but \sintvi\ and \sintviii\ do not hold. Moreover, from Proposition \ref{prop16}, it follows that $\nabla^\pi$ satisfies \sintP. Moreover, due to the freedom for building preorders over interpretations and the properties of projection and lexicographical combination, it is easy to see that $\nabla^\pi$ also satisfies \sintSD, \sintU, \sintI.
From this and Corollary \ref{cor2}, we have that it also satisfies \sintD.
\qed
\end{proof}

\begin{proposition}\label{prop21}
Let $\geq$ be a linear order over $\Val$. Then $\nabla^{\pi_\geq}$ is an ES fusion basic operator that satisfies \sintvii,\sintviii, \sintviiiw, \sintU, \sintP, \sintI\ and \sintD, but \sintv, \sintvi, and \sintSD\ do not hold.
\end{proposition}
\begin{proof}
  First we will show that $\nabla^{\pi_\geq}$ satisfies \sintvii, \sintviii\ and \sintviiiw. In order to do this, by Proposition \ref{prop4}, it is enough to prove that $\Phi\mapsto\succeq^{\pi_\geq}_\Phi$, the assignment associated to $\nabla^{\pi_\geq}$, satisfies the properties \assgiii, \assgiv\ and \assgivp\ respectively.

   To show that Property \assgiii\ holds, we note that, for all profile $\Phi$, $\succeq^{\pi_\geq}_\Phi$ is a linear order. Thus, consider $N$ in $\mathcal{F}^\ast(\mathcal{S})$, $\Set{N_1, N_2}$ a partition of $N$, let  $\Phi$ be an $N$-profile and suppose $w$, $w'$ is a pair of interpretations in $\Val$ such  that $w\succeq^{\pi_\geq}_{\Phi\upharpoonright_{N_1}} w'$ and $w\succeq^{\pi_\geq}_{\Phi\upharpoonright_{N_2}} w'$. If $w=w'$ the result follows straightforwardly. Then, suppose  that $w\neq w'$ and let us note that  $w\succ^{\pi_\geq}_{\Phi\upharpoonright_{N_1}} w'$. From this we have two cases: $w\succ_{d_{N_1}}w'$, or $w\simeq_{d_{N_1}}w'$ and $w>w'$. Thus, suppose that $w\succ_{d_{N_1}}w'$ (the case in which $w\simeq_{d_{N_1}}w'$ and $w>w'$ is similar). Since $w\succ^{\pi_\geq}_{\Phi\upharpoonright_{N_2}} w'$, we also have either $w\succ_{d_{N_2}}w'$ or $w\simeq_{d_{N_2}}w'$ and $w>w'$.  On the one hand, if $w\succ_{d_{N_2}}w'$, by virtue of $\max(N)=\max\{d_{N_1},d_{N_2}\}$, we have $w\succ_{d_{N}}w'$. Therefore  $w\succ^{\pi_\geq}_\Phi w'$. On the other hand, if $w\simeq_{d_{N_2}}w'$ and $w>w'$, we have that $w\succeq_{d_{N}}w'$. From this and the fact that $w>w'$, it follows that $w\succ^{\pi_\geq}_\Phi w'$. Therefore, $\Phi\mapsto\succeq^{\pi_\geq}_\Phi$ satisfies Property \assgiii. Similarly we obtain that Properties \assgiv\ and \assgivp\ hold.

  In order to show that a quasilinearized projective operator satisfies \sintU, we will prove that $\Phi\mapsto\succeq^{\pi_\geq}_\Phi$ satisfies property {\bf({\em u})}
  given in Proposition \ref{prop7}. Thus, consider $N$ in $\mathcal{F}^\ast(\mathcal{S})$, $\Phi$ an $N$-profile and $\succeq'$ an \ee\ such that $\succeq_i=\succeq'$, for all $i$ in $N$. Thus, $\succeq^{\mathrm{lex}(\succeq_{d_N},\geq)}=\succeq^{\mathrm{lex}(\succeq',\geq)}$, showing that $\succeq^{\pi_\geq}_\Phi=\succeq^{\pi_\geq}_{\succeq'}$.

  To see that \sintI\ holds, it is enough to see that $\Phi\mapsto\succeq^{\pi_\geq}_\Phi$ satisfies property {\bf({\em ind})} given in Proposition \ref{prop9}. Thus, consider $N$ in $\mathcal{F}^\ast(\mathcal{S})$, $\Phi$ and $\Phi'$ a pair of $N$-profiles,  and $w$, $w'$  two interpretations such that $\succeq_i\upharpoonright_{\{w,w'\}}=\succeq_i\upharpoonright_{\{w,w'\}}'$, for all $i$ in $N$. From the fact that $\succeq_{d_N}\upharpoonright_{\{w,w'\}}=\succ_{d_N}'\upharpoonright_{\{w,w'\}}$, we have $\succeq^{\mathrm{lex}(\succeq_{d_N},\geq)}\upharpoonright_{\{w,w'\}}= \succeq^{\mathrm{lex}(\succeq_{d_N}',\geq)}\upharpoonright_{\{w,w'\}}$, that is, $\succeq^{\pi_\geq}_\Phi\upharpoonright_{\{w,w'\}}=\succeq^{\pi_\geq}_{\Phi'}\upharpoonright_{\{w,w'\}}$.

  Now, since $\nabla^{\pi_\geq}$ satisfies Unanimity and Independence conditions, by Proposition \ref{prop10} we have that \sintP\ holds.

  In order to show that $\nabla^{\pi_\geq}$ satisfies \sintD, we will prove that its assignment satisfies Property {\bf({\em d})} given in Proposition \ref{prop11}. Thus, assume $N$ is in $\mathcal{F}^\ast(\mathcal{S})$, $\Phi$ is an $N$-profile and suppose that $w$, $w'$ are interpretations in $\Val$ such that $w\succ_{d_N} w'$. Then $w\succ^{\mathrm{lex}(\succeq_{d_N},\geq)}_\Phi w'$, that is $w\succ^{\pi_\geq}_\Phi w'$, as desired.

  Now, given a profile $\Phi$, since $\succeq^{\pi_\geq}_\Phi$ is a linear order over $\Val$, then, for every total preorder $\succeq$ over $\Val$, we have that $\succeq^{\mathrm{lex}(\succeq,\succeq^{\pi_\geq}_\Phi)}$ is also a linear order over $\Val$, that is,  $\nabla^{\pi_\geq}(\Phi,\succeq)$ is a linear order over interpretations. From this it follows straightforwardly that $\nabla^{\pi_\geq}$ does not satisfy \sintSD.

  To show that $\nabla^{\pi_\geq}$ does not satisfy \sintv\ and \sintvi, by virtue of Proposition \ref{prop4}, it is enough to see that $\Phi\mapsto\succeq^{\pi_\geq}_\Phi$ does not satisfy the Properties \assgi\ and \assgii\ respectively.

  To show that Property \assgi\ does not hold, consider $w$, $w'$ a pair of interpretations in $\Val$, with $w>w'$, and $\succeq_1$, $\succeq_2$ a pair
   of \ees\ satisfying that $w\simeq_1 w'$, $w\succ_2 w'$ and $\succeq_1\upharpoonright_{\Val\setminus\{w,w'\}}=\succeq_2\upharpoonright_{\Val\setminus\{w,w'\}}$.
  Thus $\succeq_1\neq\succeq_2$, but $\succeq^{\mathrm{lex}(\succeq_1,\geq)}=\succeq^{\mathrm{lex}(\succeq_2,\geq)}$. That is, $\succeq^{\pi_\geq}_{\succeq_1}=\succeq^{\pi_\geq}_{\succeq_2}$.

  Finally, to see that Property \assgii\ does not hold, consider $w$, $w'$ a pair of interpretations in $\Val$ such that $w>w'$, $N=\{1,2\}$
  and $\Phi=(\succeq_1,\succeq_2)$ an $N$-profile  with $\max(\succeq_i)=\{w,w'\}$, for $i=1,2$.
  Thus, $\model{\Bigwedge_{i\in N} B(\succeq_i)}=\{w,w'\}$ and by definition we have $\max(\succeq^{\pi_\geq}_\Phi)=\{w\}$. This shows that $\model{\Bigwedge_{i\in N} B(\succeq_i)}\neq\max(\succ^{\pi_\geq}_\Phi)$.\qed

\end{proof}

\begin{proposition}\label{prop20}
Let $\geq$ be a linear order over $\Val$. Then $\nabla^{\mathrm{Q}\pi_\geq}$ is an ES fusion basic operator that satisfies \sintv, \sintviiiw, \sintSD, \sintP, \sintI\ and \sintD, but \sintvi, \sintvii, \sintviii, and \sintU\ do not hold.
\end{proposition}
\begin{proof}
 The proof that $\nabla^{\mathrm{Q}\pi_\geq}$ satisfies \sintviiiw, \sintI\ and \sintD\, but  not \sintvi,
 is similar to that in Proposition \ref{prop20}. Now, in order to show that $\nabla^{\mathrm{Q}\pi_\geq}$ satisfies \sintv,
 by Proposition \ref{prop4}, it is enough to see that  $\Phi\mapsto\succeq^{\mathrm{Q}\pi_\geq}_\Phi$, the assignment associated
 to this operator, satisfies \assgi, but this  follows straightforwardly from the fact that $\Phi\mapsto\succeq^{\mathrm{Q}\pi_\geq}_\Phi$
 is structure preserving. Due to the freedom for building total preorders over $\Val$ and the definition of $\nabla^{\mathrm{Q}\pi_\geq}$
 it follows that it satisfies \sintSD. Moreover, since $\nabla^{\mathrm{Q}\pi_\geq}$ is a dictatorial operator, by Proposition \ref{prop12},
 we have that \sintP\ holds.

In order to prove that \sintvii\ and \sintviii\ do not hold, by Proposition \ref{prop4}, it is enough to
 see that $\Phi\mapsto\succeq^{\mathrm{Q}\pi_\geq}_\Phi$ does not satisfy Properties \assgiii\ and \assgiv\ respectively.
 Consider the finite society $N=\{1,2\}$, and its partition $N_1=\{1\}$, $N_2=\{2\}$. Let  $w$, $w'$ be a pair of interpretations in $\Val$  such that $w'>w$ and $\Phi=(\succeq_1,\succeq_2)$ be an $N$-profile such that  $w\succeq_1 w'$, $w\simeq_2 w'$. From the definition of the assignment we have $w\succ^{\mathrm{Q}\pi_\geq}_{\Phi\upharpoonright_{N_1}}w'$ and $w\succeq^{\mathrm{Q}\pi_\geq}_{\Phi\upharpoonright_{N_2}}w'$, but  $w'\succ^{\mathrm{Q}\pi_\geq}_\Phi w$.

Finally, to show that \sintU\ does not hold, consider $w$, $w'$ a pair of interpretations in $\Val$, such that $w>w'$.
Let $\succeq^\ast$ be a total preorder over $\Val$ such that $\max(\succeq^\ast)=\{w,w'\}$.
Define $\Phi=(\succeq_1,\succeq_2)$ by putting $\succeq_i=\succeq^\ast$,  for $i=1,2$.
Thus, by definition $w\succ^{\mathrm{Q}\pi_\geq}_\Phi w'$ and $\succeq^{\mathrm{Q}\pi_\geq}_{\succeq^\ast}=\succeq^\ast$ and, therefore, $\succeq^{\mathrm{Q}\pi_\geq}_\Phi\neq\succeq^{\pi_\geq}_{\succeq^\ast}$. The result follows from Proposition \ref{prop7}.\qed
\end{proof}

\begin{proposition}\label{prop18}
  $\nabla^{\Sigma-\mathrm{P}\pi}$  is an ES basic fusion operator that satisfies \sintv, \sintvi, \sintSD, \sintU, \sintP, and  \sintD, but does not satisfy \sintvii, \sintviii, \sintviiiw\ nor \sintI.
\end{proposition}

\begin{proof}
In order too see that $\nabla^{\Sigma-\mathrm{P}\pi}$ satisfies \sintv\ and \sintvi, by Proposition \ref{prop4}, it is enough to see that
$\Phi\mapsto \succeq_\Phi^{\Sigma-\mathrm{P}\pi}$ satisfies Properties \assgi\ and \assgii.
Assume $N$ in $\mathcal{F}^\ast(\mathcal{S})$ and consider  an $N$-profile $\Phi$. Let  $w$ and $w'$ be  a pair of interpretations. Since $\succeq_\succeq^\Sigma=\succeq$ and $\succeq^{\mathrm{lex}(\succeq,\succeq)}=\succeq$ we have that $\succeq_\succeq^{\Sigma-\mathrm{P}\pi}=\succeq$. Thus $\Phi\mapsto\succeq_\Phi^{\Sigma-\mathrm{P}\pi}$ is structure preserving. Therefore, it satisfies Property \assgi.

In order to show that Property \assgii\ holds, we first suppose that $w$ is a model of $\Bigwedge_{i\in N} B(\succeq_i)$. Thus,
$w\models B(\succeq_{d_N})$ and by the maximality condition we have that $w\succeq_{d_N}w'$.
Since   $\Phi\mapsto\succeq_\Phi^\Sigma$ satisfies \assgii, $w\succeq_\Phi^\Sigma w'$.
Thus, $w\succeq_\Phi^{\Sigma-\mathrm{P}\pi}w'$, that is, $w$ is in $\max(\succeq_\Phi^{\Sigma-\mathrm{P}\pi})$.
Suppose now, towards a contradiction, that $w$ is in $\max(\succeq_\Phi^{\Sigma-\mathrm{P}\pi})$
and $w\not\models\Bigwedge_{i\in N} B(\succeq_i)$. Thus, if $w'\models\Bigwedge_{i\in N} B(\succeq_i)$,  by \assgii\ we have that
$w'\succ_\Phi^\Sigma w$. On the other hand,  since $w\succeq_\Phi^{\Sigma-\mathrm{P}\pi} w'$ it
satisfies  either $w\succ_{d_N}w'$ (in this case, by the maximality condition, $w'\not\models B(\succeq_{d_N})$) or $w\simeq_{d_N}w'$
and $w\succeq_\Phi^\Sigma w'$. In both cases we get a contradiction. Thus Property \assgii\ is proved.

\begin{figure}[t]
\begin{center}
\includegraphics{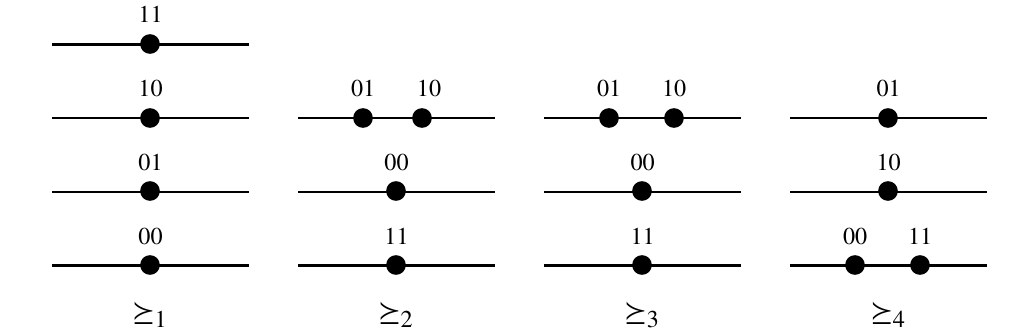}
\end{center}
\caption{The profile of epistemic states.}\label{fig2.6}
\end{figure}

In order to see that $\nabla^{\Sigma-\mathrm{P}\pi}$ does not satisfy \sintvii, \sintviii\ and \sintviiiw, by Proposition \ref{prop4}, it is enough to see that
$\Phi\mapsto \succeq_\Phi^{\Sigma-\mathrm{P}\pi}$ does not satisfy Properties \assgiii, \assgiv\ and \assgivp.
Consider the set of interpretations $\mathcal{W_P}=\{00,01,10,11\}$ of a finite propositional language $\mathcal{L_P}$ with two propositional variables.
Let $N=\Set{1,2,3,4}$. Let $\Phi$ be the     $N$-profile  represented in figure \ref{fig2.6}. Consider the partition of $N$ given by $N_1=\{1,2\} $, $N_2=\{3,4\}$. Let us note the following:
\begin{itemize}
  \item Since $00\succ_2 11$,  we have $00\succ^{\Sigma-\mathrm{P}\pi}_{\Phi\upharpoonright_{_{N_1}}} 11$
  \item Since $00\succ_3 11$ and $00\simeq_4 11$ then $00\succ^{\Sigma}_{\Phi\upharpoonright_{_{N_2}}} 11$. Hence $00\succ^{\Sigma-\mathrm{P}\pi}_{\Phi\upharpoonright_{_{N_2}}} 11$
\end{itemize}
However,  since $11\simeq_4 00$ and easy calculations lead to $11\succ_\Phi^\Sigma 00$, we have that  $11\succ_\Phi^{\Sigma-\mathrm{P}\pi}00$. Thus, \assgiii, \assgiv\ and \assgivp\ do not hold.

 Moreover, due to the freedom for building total preorders and because the basic assignment preserves the structure of epistemic states, we have that the $\Sigma$-pseudoprojective ES basic fusion operator satisfies \sintSD.

In order to show that \sintU\ holds, consider $N$ in $\mathcal{F}^\ast(\mathcal{S})$, an $N$-profile $\Phi$, $\succeq$
an epistemic state such that $\succeq_i=\succeq$ for each $i$ in $N$. Thus $\succeq_\Phi^\Sigma=\succeq'$, and
therefore $\succeq_\Phi^{\Sigma-\mathrm{P}\pi}=\succeq^{\mathrm{lex}(\succeq,\succeq)}=\succeq$. From this, the result follows
using   Proposition \ref{prop7}.

Straightforward from the definition, we can see that $\nabla^{\Sigma-\mathrm{P}\pi}$ has a dictator. Thus, \sintD\ holds. Then, by Proposition \ref{prop12},
\sintP\ holds.

We know that $\nabla^{\Sigma}$ does not satisfy \sintI. Thus, by Proposition \ref{prop9}, there exists a society $D$, two $D$-profiles $\Phi$ and $\Phi'$
and two interpretations $w$ and $w'$ such that $\succeq_i\upharpoonright_{\Set{w,w'}}=\succeq'_i\upharpoonright_{\Set{w,w'}}$ for each $i\in D$, but
$w\succ^\Sigma_\Phi w'$ and $w'\succeq^\Sigma_{\Phi'} w$.
 Choose an agent $j$ such that $j>d_D$.
 Consider now the society $N=D\cup\Set j$ and let $\succeq_j$ be the flat preorder (all interpretations are indifferent). Let $\Phi_1=\Phi\sqcup \succeq_j$ and $\Phi_2=\Phi'\sqcup \succeq_j$. It is clear that
$\succeq_i\upharpoonright_{\Set{w,w'}}=\succeq'_i\upharpoonright_{\Set{w,w'}}$ for each $i\in N$. We can also check that
$\succeq_{\Phi_1}^{\Sigma-\mathrm{P}\pi}=\succ^\Sigma_\Phi$ and $\succeq_{\Phi_2}^{\Sigma-\mathrm{P}\pi}=\succ^\Sigma_{\Phi'}$. Therefore,
$w\succ^{\Sigma-\mathrm{P}\pi}_{\Phi_1} w'$ and $w'\succeq^{\Sigma-\mathrm{P}\pi}_{\Phi_2} w$. That is, again by Proposition \ref{prop9}, \sintI\ does not hold.
\qed
\end{proof}

\end{document}